\documentclass[11pt]{article}

\usepackage{times}
\usepackage{natbib}
\usepackage{amsthm,amsfonts,amsmath,amssymb}
\usepackage{algorithm}
\usepackage{algorithmicx}
\usepackage{algpseudocode}
\usepackage{booktabs}
\usepackage{hyperref}
\usepackage{fullpage,url}
\usepackage{makecell}
\usepackage{graphicx}
\usepackage{pifont}
\usepackage{tikz-cd}
\usepackage{wrapfig}
\usepackage{multirow}
\usepackage{enumitem}
\usepackage{caption}
\usepackage{subcaption}
\usepackage{bbding}

\hypersetup{
	colorlinks   = true, 
	urlcolor     = blue, 
	linkcolor    = blue, 
	citecolor   = blue, 
    breaklinks=true
}

\setcitestyle{authoryear,round,citesep={;},aysep={,},yysep={;}}

\newtheorem{theorem}{Theorem}
\newtheorem{lemma}{Lemma}
\newtheorem{proposition}{Proposition}

\newtheorem{assumption}{Assumption}

\title{From Generalization Analysis to Optimization Designs for State Space Models}
\author
{
Fusheng Liu\\
National University of Singapore\\
{\tt fusheng@u.nus.edu} 
\and
Qianxiao Li\\
National University of Singapore\\
{\tt qianxiao@nus.edu.sg}
}
\date{}

\begin{document}
\maketitle

\begin{abstract}
A State Space Model (SSM) is a foundation model in time series analysis, which has recently been shown as an alternative to transformers in sequence modeling.  
In this paper, we theoretically study the generalization of SSMs and  
propose improvements to training algorithms based on the
generalization results.
Specifically, we give a \textit{data-dependent} generalization bound for SSMs,
showing an interplay between the SSM parameters and the
temporal dependencies of the training sequences.
Leveraging the generalization bound,
we 
(1) set up a scaling rule for model initialization based on the proposed generalization measure,
which significantly improves the robustness of the output value scales on SSMs
to different temporal patterns in the sequence data;
(2) introduce a new regularization method for training SSMs to enhance the generalization performance.
Numerical results are conducted to validate our results.
\end{abstract}

\section{Introduction}

Sequence modeling has been a long-standing research topic in many machine learning areas, such as speech recognition \citep{hinton2012deep}, 
time series prediction \citep{li2019enhancing}, and natural language processing \citep{devlin-etal-2019-bert}.
Various machine learning models have been successfully applied in sequence modeling to handle different types of sequence data, ranging from the (probabilistic) Hidden Markov model \citep{baum1966statistical} to deep learning models, e.g., 
Recurrent Neural Networks (RNNs),
Long Short-Term Memory units \citep{hochreiter1997long}, 
Gated Recurrent Unit \citep{chung2014empirical},
and transformers \citep{vaswani2017attention}.
In this paper, we focus on the state space model (SSM), which has a simple mathematical expression: $h'(t) = Ah(t) + Bx(t), y(t) = Ch(t) + Dx(t)$
where $h(t)$ is the hidden state, $x(t)$ is the input sequence, $y(t)$ is the output sequence and $A, B, C, D$ are trainable parameters. 
To simplify the analysis, we omit the skip connection by letting $D = 0$. 
In fact, our analysis can also applied to the case when $D$ is included (see the discussions in Section \ref{section: initialization scheme}).
Recent studies have demonstrated the power of SSMs in deep learning.
For example, it was shown in \citet{gu2022efficiently} that by a new parameterization and a carefully chosen initialization, the structured state space sequence (S4) model achieved strong empirical results on image and language tasks.
Following the S4 model, more variants of SSMs are proposed, e.g., 
diagonal SSMs \citep{gu2022s4d, gupta2022diagonal}
S5 \citep{smith2023simplified}, 
H3 \citep{fu2023hungry},
GSS \citep{mehta2023long},
Hyena Hierarchy \citep{poli2023hyena},
and Mamba \citep{gu2023mamba}.

Theoretical analysis and understanding of the approximation and optimization of SSMs are well studied in the literature such as \citep{li2021on, li2022approximation, gu2022efficiently, gu2023how}.
Since the SSM can be regarded as a continuous linear RNN model \citep{li2022approximation}, most generalization analysis of SSMs is based on the generalization theory of RNNs \citep{zhang2018stabilizing, chen2019generalization, tu2019understanding}.
However, these previous works did not study the effects of the temporal dependencies in the sequence data on the SSM generalization (see more details on the comparison in Section \ref{section: generalization bound}).
As an attempt to understand the relationship between the temporal dependencies and the generalization performance, this paper
aims to provide a generalization bound that connects the memory structure of the model with the temporal structure of the data.
We can, in turn, use the proposed bound to guide us in designing new algorithms to improve optimization and generalization.
Specifically, we discover two roles for the proposed generalization measure: 
(1) generalization bound as an \textit{initialization scheme};
(2) generalization bound as a \textit{regularization method}.
The common initialization method for the S4 model and its variants follows from the HiPPO framework \citep{gu2022efficiently, gu2023how}, which is based on the prerequisite that the training sequence data is stable.
To improve the robustness of the output value scales on SSMs
to different temporal patterns in the sequence data, we consider to rescale the initialization of SSMs with respect to the generalization measure.
This new initialization scheme makes the SSMs more resilient
on their initial output value scales
to variations in the temporal patterns of the training data.
Except for the initialization setup, our generalization bound can also be served as a regularizer.
Regularization methods like weight decay and dropout are widely applied to training SSMs, but the hidden state matrix $A$ is not regularized because its imaginary part controls the oscillating frequencies of the basis function $e^{A t} B$ \citep{gu2022s4d}.
By taking into account the 
interaction between the SSM structure and the temporal dependencies, we introduce a new regularization method based on our bound, and it can be applied to the hidden state space to improve the generalization performance.
Combining the initialization scheme and the regularization method, our method is applicable to various tasks, ranging from image classification to language processing, while only introducing a minimal computational overhead.
To summarize, our contributions are as follows: 
\begin{itemize}
    \item We provide a data-dependent generalization bound for SSMs by taking into account the temporal structure.
    Specifically, the generalization bound
    correlates with the memory structure of the model and the (auto)covariance process of the data.
    It indicates that 
    instead of the weight or the data norm,
    it is the interplay between the memory structure and the temporal structure of  the sequence data that influences the generalization.
    \item 
    Based on the proposed generalization bound, we 
    setup an initialization scaling rule by adjusting the 
    magnitude of the model parameters with respect to the generalization measure at initialization.
    This scaling rule 
    improves the robustness
    of the initial output value scales on SSMs
    across different temporal patterns of the sequence data.
    \item 
    Apart from the initialization scheme, we design a new regularizer for SSMs.
    Unlike weight decay, our regularizer does not penalize the parameter norm but encourages the model to find a minimizer with lower generalization bound to improve the generalization performance.
\end{itemize}

\section{Related Works}

Since a SSM is also a continuous linear RNN, there are three lines of related work: 
generalization of RNNs,
temporal structure analysis on RNNs,
and optimization of SSMs.

\textbf{Generalization of RNNs.}
Existing works on the generalization of RNNs focus on the generalization error bound analysis.
Specifically, in the early two works of \citet{dasgupta1995sample} and \citet{koiran1998vapnik}, 
VC dimension-based generalization bounds were provided to show the learnability of RNNs. 
In recent studies, \citet{zhang2018stabilizing,
chen2019generalization, tu2019understanding} 
proved norm-based generalization bounds, improving the VC dimension-based bounds by the Rademacher complexity technique \citep{bartlett2002rademacher} under the uniform-convergence framework.
In the overparameterization settings, it was shown in \citet{allen2019can} that RNNs can learn some concept class in polynomial time given that the model size is large enough.
These generalization bounds, however, do not take into account the temporal dependencies and their effects on generalization.
In this work, we provide a new generalization bound by combining the memory structure of the model and the temporal structure of the data.

\textbf{Temporal structure analysis on RNNs.}
Sequence data has long-range temporal dependencies across the time domain, which notably set it apart from non-sequence data.
Recent studies have studied the effects of such temporal dependencies on the approximation and optimization of RNNs.
For example, in the two works of \citet{li2021on, li2022approximation}, a ``curse of memory'' phenomenon was discovered when using linear RNNs to model the  temporal input-output relationships.
Particularly, when the target relationship between the input and output has a long-term memory, then both approximation and optimization become  extremely challenging. 
In \citet{wang2023inverse}, the ``curse of memory'' phenomenon on approximation and optimization was extended to non-linear RNNs 
based on the temporal relationships.
In this paper, we conduct a fine-grained analysis on the effects of the temporal structure analysis on the \textit{generalization} of RNNs.

\textbf{Optimization of SSMs.} 
RNN optimization is known for two issues: training stability and computational cost \citep{bengio1994learning, pascanu2013difficulty}.
To address these issues and capture the long dependencies efficiently in sequence modeling, the S4 model was proposed by new paraemterization, initialization and discretization \citep{gu2022efficiently}. 
Recent variants for the S4 model simplified the hidden state matrix by a diagonal matrix to enhance computational efficiency \citep{gu2022s4d, gupta2022diagonal, smith2023simplified, orvieto2023resurrecting}.
Regularization methods are also applied for SSMs to prevent overfitting, such as dropout, weight decay and the data continuity regularizer \citep{qu2023data}.
However, the principled way to regularize and initialize the parameters still remains to be explored.
In this study, we design a new regularization and initialization scheme to improve both optimization and generalization.

\section{Preliminaries}

In this section, we briefly introduce the SSM in Section \ref{section: introduction} and the motivation for optimization designs based on the generalization analysis in Section \ref{section: motivation}.

\subsection{Introduction to SSMs}\label{section: introduction}

In this paper, we consider the following single-input single-output SSM, 
\begin{equation}\label{eq: def SSM}
    h'(t) = Ah(t) + Bx(t), \quad y(t) = C h(t), \quad
    t \geq 0
\end{equation}
where $x$ is the input from an input space\footnote{A linear space of continuous functions from $\mathbb{R}_{\geq 0}$ to $\mathbb{R}$ that vanishes at
infinity.} $\mathcal{X} := C_0(\mathbb{R}_{\geq 0}, \mathbb{R})$;
$y(t) \in \mathbb{R}$ is the output at time $t$;
$h(t) \in \mathbb{R}^m$ is the hidden state with $h(0) = 0$;
$A \in \mathbb{R}^{m \times m}, B \in \mathbb{R}^{m \times 1}, C \in \mathbb{R}^{1 \times m}$ are trainable parameters.
Then (\ref{eq: def SSM}) has an explicit solution $y(t) = \int_0^t \rho_\theta(s) x(t-s) d s$,
where $\rho_\theta(s) := C e^{A s} B$ with $\theta = (C, A, B)$.
The function $\rho_\theta(s)$ captures the memory structure of the model and the temporal input-output relationship \citep{li2022approximation}.
For the S4 model and its variants \citep{gu2022efficiently, gu2022s4d, gupta2022diagonal, gu2023how},
(\ref{eq: def SSM}) is usually discretized by the Zero-Order Hold method, i.e.,
given a timescale $\Delta \in \mathbb{R}$,  $h_{k+1} = \Bar{A} h_k + \Bar{B} x_k, \quad 
    y_k = \Bar{C} h_k, \quad
    k = 0,1,\ldots,$
where $\Bar{A} = e^{\Delta \cdot A}, \Bar{B} = (\Bar{A} - \mathbb{I}_m) A^{-1} B, \Bar{C} = C$.
Then, $y_k = \Bar{C} \Bar{A}^k \Bar{B} x_0 + \Bar{C} \Bar{A}^{k-1} \Bar{B} x_1 + \ldots + \Bar{C} \Bar{B} x_k = [\Bar{K} * x]_k$ where $\Bar{K} = (\Bar{C} \Bar{B}, \Bar{C} \Bar{A} \Bar{B}, \ldots, \Bar{C} \Bar{A}^{k} \Bar{B})$ and $*$ represents to convolution.

\subsection{Motivation: a linear regression model}\label{section: motivation}

In this subsection, we use a linear regression model on non-sequential data as an example to illustrate the connection between the generalization analysis and the optimization designs.
This example then motivates us to extend the connection to SSMs on sequential data.

\textbf{Linear regression.}\ 
We consider a simple linear model $y = \theta^\top x$ with input $x \in \mathbb{R}^d$, output $y \in \mathbb{R}$ and parameter $\theta \in \mathbb{R}^d$.
Let the training data $\{(x_i, y_i)\}_{i=1}^n$ be i.i.d. sampled from a distribution $\mathcal{D}$ such that $\|x_i\|_2 = r, |y_i| \leq 1 (\forall i \in [1:n])$.
Define the empirical risk $\mathcal{L}_n(\theta) := \frac{1}{n} \sum_{i=1}^n (\theta^\top x_i - y_i)^2$ and the population risk $\mathcal{L}_\mathcal{D} (\theta) := \mathbb{E}_{x,y}[(\theta^\top x - y)^2]$.
Then given a norm-constrained space $\Theta := \{\theta \in \mathbb{R}^d: \|\theta\|_2 \leq R\}$, with probability at least $1-\delta$ over $\mathcal{D}$, 
\begin{equation}\label{eq: bound for linear regression}
    \sup_{\theta \in \Theta} |\mathcal{L}_n(\theta) - \mathcal{L}_\mathcal{D} (\theta)| \leq (r R + 1)^2 \cdot \mathcal{O}
    (\sqrt{\log(1/\delta)/n}).
\end{equation}
This is a well-known norm-based generalization bound  based on the Rademacher theory \citep{Mohri}, and we provide a proof in Appendix \ref{appendix, proof motivation} for completeness. 
Notice that the key term $r^2 R^2$ in the generalization bound (\ref{eq: bound for linear regression}) is also an upper bound for the magnitude of the linear model output, i.e., $\sup_{\theta \in \Theta} (\theta^\top x_i)^2 \leq r^2 R^2$.
Thus, we connect the model stability with the generalization bound stability, and this connection induces an initialization scheme for the initialization $\theta^{(0)}$ by 
setting $\|\theta^{(0)}\|_2 \sim \mathcal{O}(1/r)$.
In particular, if we normalize each input $x_i$ such that $r$ is also $\mathcal{O}(1)$, then $\|\theta^{(0)}\|_2 \sim \mathcal{O}(1)$.
Since $\theta^{(0)} \in \mathbb{R}^d$, one possible initialization scheme is that $\theta^{(0)}$ follows a Uniform distribution $U[-1/\sqrt{d}, 1/\sqrt{d}]$, 
which corresponds to the Kaiming initialization (up to some constant) \citep{he2015delving}.
When treating the term $r^2 R^2$ as a regularizer to improve the generalization, we get the weight decay method, i.e., the $\ell_2$ regularization w.r.t. $\|\theta\|_2^2$.
We summarize the above logic chain that connects the generalization analysis with optimization designs in Figure \ref{fig: logic diagram}.
\begin{figure}
    \centering
    \begin{tikzcd}
    \label{logic chain}
    & & \text{Initialization scheme: set $\theta_0$ s.t. $\text{Complexity}(\theta_0) = \mathcal{O}(1)$} \\
     & 
     \begin{tabular}{c}
          \text{Generalization Estimate}  \\
          $\text{GenError}(\theta) \sim \frac{\text{Complexity}(\theta)}{\sqrt{n}}$ 
     \end{tabular}
     \arrow[ru, Rightarrow] \arrow[rd, Rightarrow] & \\
    & & \text{Regularization method: penalize $\text{Complexity}(\theta)$}
\end{tikzcd}
\caption{The logic diagram goes from generalization analysis to optimization designs.}
\label{fig: logic diagram}
\end{figure}
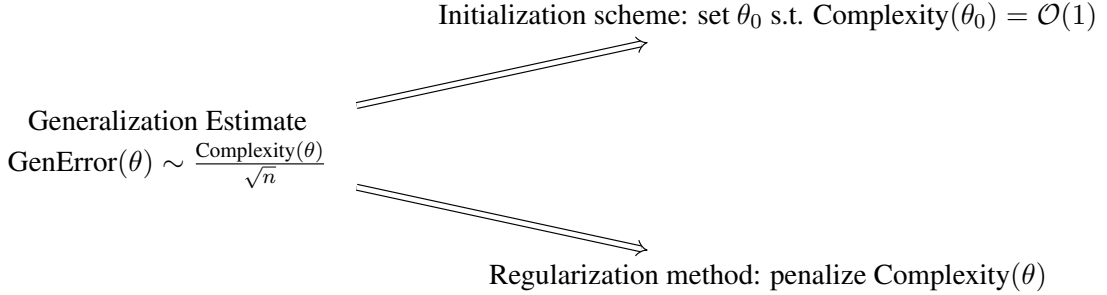
Now for SSMs, we extend the generalization analysis from non-sequential data to sequential data by taking into account the temporal structure of the data.
This linear regression example motivates us to apply the same logic diagram (Figure \ref{fig: logic diagram}) to the SSMs, and this is exactly what we are going to present in the following part of this paper.

\section{Main results}

In this section, we first  give a generalization bound for SSMs in Section \ref{section: generalization bound}, then we design a new initialization scheme in Section \ref{section: initialization scheme} based on this proposed bound.
Apart from the initialization scheme, we introduce a new regularization method in Section \ref{section: regulaization method}.
Finally, we conduct experiments to test the initialization scheme and the regularization method in Section \ref{section: experiment}.

\subsection{A generalization bound of SSMs}\label{section: generalization bound}

In this section, we present a generalization bound for the SSM (\ref{eq: def SSM}) and reveal the effects of the temporal dependencies on the generalization performance.
We show that our bound gives a tighter estimate compared with previous norm-based bounds through a toy example.
Following the same notation in Section \ref{section: introduction}, we define the empirical risk ${R}_n(\theta)$ and the population risk $R_{x}(\theta)$ as 
\begin{equation*}
    {R}_n(\theta) := 
        \frac{1}{n} \sum_{i=1}^n
    \left|\int_0^T  {\rho}_\theta (T-s)x_i(s) d s - y_i\right|^2, \quad
    R_{x}(\theta)  := \mathbb{E}_{x} \left|\int_0^T  {\rho}_\theta (T-s) x(s) d s - y\right|^2
\end{equation*}
where $T>0$ is some finite terminal time, the training sequence data $\{x_i(t)\}_{i=1}^n$ are independently sampled from a 
stochastic process with mean $\mathbb{E}[x(t)] := \mu(t)$ and covariance $\mathbb{E}[(x(s) - \mu(s))(x(t) - \mu(t))] := K(s, t)$,
and the label $y$ is generated by some underlying functional $H_T : \mathcal{X} \xrightarrow{} \mathbb{R}$, i.e., $y = H_T(x)$.
We assume that $|y| \leq 1$ for any $x \in \mathcal{X}$, otherwise, we truncate the value of the label to $1$.
In the next, we make an assumption on the
normalized process $\Tilde{x}(t) := 
(x(t)-\mu(t)) / \sqrt{K(t,t)}$:
\begin{assumption}\label{assumption: as finite}
    The normalized process $\Tilde{x}(t)$ is (1): almost surely Hölder continuous, i.e., 
    $\exists L, H > 0, s.t. \forall s, t \in [0, T], |\Tilde{x}(s) - \Tilde{x}(t)| \leq L |s-t|^{H} a.s.$;
    (2): is $\sigma^2$-sub-Gaussian for every $t \in [0, T]$, i.e., 
    $\exists \sigma > 0, s.t. \forall u > 0, P \left(|\Tilde{x}(t)| \geq  u\right) \leq 2\exp (-u^2/2\sigma^2)$ for any $t \in [0, T]$.
\end{assumption}
We leave the discussion of the assumption after the statement of the main theorem.
Now we proceed to bound generalization gap $|R_x(\theta) -  {R}_n(\theta)|$ by establishing uniform convergence of the empirical risk to its corresponding population risk, as stated in following theorem:
\begin{theorem}\label{thm: p=1, q=inf}
    For a SSM $\int_0^T  {\rho}_\theta(T-s) x(s) d s$, following notations and settings in Section \ref{section: introduction} \& \ref{section: generalization bound},
    we define $\psi(\Theta) := \sup_{\theta \in \Theta} \int_0^T \left| {\rho}_\theta(T-s)\right|
    \sqrt{K(s, s)}
    d s 
    + 
    \sup_{\theta \in \Theta} \left|\int_0^T  {\rho}_\theta(T-s) \mu(s)
    d s\right|$.
    Then under Assumption 
    \ref{assumption: as finite},
    given a parameter space $\Theta$ for $\theta$, for any $\delta \in (0, 1)$, with probability at least $1-\delta$ over the training sequences,
\begin{equation}\label{thm: bound}
    \sup_{\theta \in \Theta} \left|R_x(\theta) -  {R}_n(\theta)\right| 
        \leq 
    (\psi(\Theta)+1)^2
    \cdot 
    \mathcal{O} (\log^{3/2}(Tn/\delta)/{\sqrt{n}}).
\end{equation}
\end{theorem}
Where $\mathcal{O}$ hides a constant that depends on $\sigma, L, H$.
The proof is given in Appendix \ref{section: proof for thm1}. 
We see that this bound decreases to zero as the sample size $n \xrightarrow{} \infty$, provided that the terminal time $T$ is finite and $\psi_{\Theta}$ grows slower than $\sqrt{n}$.
For example, when the data statistics (e.g., $\mu(s)$ and $K(s,s)$) are uniformly bounded along the time horizon, by the exponentially decay property of the SSM function $\rho_\theta(s)$, we have $\psi_\Theta$ is finite, then the generalization bound is $\Tilde{\mathcal{O}}(1/\sqrt{n})$,
yielding that the mean and variance at each length position together play important roles in generalization analysis.

\textbf{Proof sketch.}\ 
The proof is based on Rademacher theory \citep{bartlett2002rademacher}.
The main difficulty is to bound the Rademacher complexity of the SSM function $\int_0^T {\rho}_\theta(T-s) x(s) d s$ for a stochastic process $x(s)$.
We first use the Hölder inequality to get an upper bound for the Rademacher complexity w.r.t. the normalized process $\Tilde{x}(s)$, then combining Hölder continuity and the heavy-tail property in Assumption \ref{assumption: as finite}, we show the finiteness of $\sup_{s \in [0, T]}\Tilde{x}(s)$. 
Finally we use an $\varepsilon$-net argument to give an explicit bound for the Rademacher complexity, which then finishes the proof.

\textbf{Discussions of Assumption \ref{assumption: as finite}.}\
This assumption contains two parts.
Hölder continuity is used to 
bound $\sup_{s \in [0, T]}\Tilde{x}(s)$ and the Rademacher complexity of the SSM function class.
By the Kolmogorov continuity theorem \citep{stroock1997multidimensional}, Hölder continuity covers a wide range of random process that satisfies certain inequalities for its moments.
For the sub-Gaussian property, it ensures $\Tilde{x}(s)$ is bounded in a finite time set with high probability.
Sub-Gaussian random variables include Gaussian and any bounded variables.
Specifically, for image classification tasks with flattened image pixels, if the range of the pixel values is a finite class (e.g., integer numbers from 0 to 255), then the Hölder continuity condition can be dropped. 
We leave more detailed discussions and provide some concrete examples that satisfy Assumption \ref{assumption: as finite} in Appendix \ref{section: discuss assumption 1}.

\textbf{Comparison to previous bounds.}\ 
Since a SSM is also a continuous linear RNN model, we compare (\ref{thm: bound}) with previous bounds for linear RNNs.
A generalization bound $\widetilde{\mathcal{O}} \left(\|x\|_2  \|B\|_2  \|C\|_2 \|A\|_2 / {\sqrt{n}}\right)$ is provided In \citet{chen2019generalization}, where $\|x\|_2$ is the 2-norm of the discrete input sequence.
In the continuous case, $\|x\|_2$ corresponds to the $L^2$ norm w.r.t. a Dirac measure.
By changing the matrix 2-norm to matrix 1-norm, \citet{tu2019understanding} shows another similar generalization bound. 
These bounds separate the data complexity and the model complexity by the data norm and the model parameter norm individually, and do not account for the temporal dependencies across the time domain.
In this work, instead, we incorporate the temporal dependencies via the sequence statistics (mean and variance) to get a generalization bound.
Next, we use a toy example to illustrate that our bound gives a tighter estimation.
Given a stochastic process $\{x(t)\}_{t \in [0, T]}$ with mean $\mu(t)$ and covariance $K(s,t)$, we consider the following two upscale transformations (by increasing $T$ to $2T$):
\begin{enumerate}
    \item 
    left zero padding: 
    $x_1(t)  = 0, \ t \in [0, T);
    \quad
    x_1(t)  = x(t-T), \ t \in [T, 2T]$
    \item 
    right zero padding: 
    $x_2(t)  = x(t), \ t \in [0, T]; 
    \quad
    x_2(t)  = 0, \ t \in (T, 2T]$
\end{enumerate}
Then the two SSM outputs are given by $y_i(2T) = \int_0^{2T}  {\rho}_\theta(2T-s) x_i(s) ds$ for $i=1,2$.
Hence,
\begin{equation*}
    y_1(2T) = C
     \int_0^{T} e^{A(T-s)}B x(s) ds, 
    \quad
    y_2(2T) = C e^{AT} 
    \int_0^{T} e^{A(T-s)}B x(s) ds.
\end{equation*}
We see that the magnitude of $y_1(2T)$ and $y_2(2T)$ differs with an exponential factor $e^{AT}$.
Since all the eigenvalues of $A$ have negative real part, $y_2(2T) \xrightarrow{} 0$ as $T$ increases.
Hence, the right zero padding transformation degenerates the SSM function class to a zero function class for large $T$, inducing a \textit{minimal} generalization gap that only contains the statistical sampling error (see (\ref{thm: bound}) by letting $K(s,s) = \mu(s) = 0$).
Therefore, a desired generalization bound should reflect such a difference caused by the different temporal dependencies.
However, previous norm-based generalization bounds do not capture such a difference for these two transformations as they 
produce the same $L^2$ norm for the input sequence.
Let us see what happens for our proposed generalization measure.
For the left zero padding, the key term in (\ref{thm: bound}) becomes 
\begin{equation}\label{measure for left zero padding}
     \int_0^{T} \left|C e^{A(T-s)} B\right| \sqrt{K(s, s)} d s +
    \left|\int_0^{T} C e^{A(T-s)} B \mu(s)
    d s\right|+1
\end{equation}
For the right zero padding, the key term in (\ref{thm: bound}) becomes 
\begin{equation}\label{measure for right zero padding}
    \int_0^{T} \left|C e^{AT} e^{A(T-s)}B\right| \sqrt{K(s, s)} d s +
    \left|\int_0^{T} C e^{AT} e^{A(T-s)}B \mu(s)
    d s\right|+1
\end{equation}
The detailed derivations are given in Appendix \ref{appendix, derivation for padding}.
By the same argument, our bound (\ref{thm: bound}) indeed captures the difference on the magnitude of the generalization performance for these two sequence transformations.
In particular, as $T \xrightarrow{} \infty$, 
(\ref{measure for right zero padding}) reduces to $1$, which yields a minimal generalization gap as expected for the zero function class.
In that sense, we get a tighter bound for the SSMs.

\textbf{Zero shot transferability.}\
A benefit of SSMs is the zero-shot transferability to other sampling frequencies (i.e., the timescale measure in continuous case).
For example, for a SSM function $y_T = \int_0^T  {\rho}_\theta(T-s) x(s) d s$, if we downscale the input sequence $x(s)$ by half of the sampling frequency, then the SSM output becomes $y_T = \int_0^{T/2}  {\rho}_\theta(T-2s) x(2s) ds$, which equals to $\int_0^{T} \frac{1}{2} {\rho}_\theta(T-s) x(s) ds$.
Now for a new SSM parameter $\Tilde{\theta} = (2C, A, B)$, we have $\rho_{\Tilde{\theta}}(s) = 2 \rho_\theta(s)$, indicating that by simply modifying the SSM parameters, one can transfer the model to half the sampling frequency while keeping the output invariant.
One advantage for our generalization measure is that it is also zero shot transferable.
To see this, we use the same example here.
Under the downscale sampling, both 
$\int_0^T \left| {\rho}_\theta(T-s)\right| \sqrt{K(s, s)}d s$ and $\left|\int_0^T  {\rho}_\theta(T-s) \mu(s) d s\right|$ remain invariant for the new parameter $\Tilde{\theta}$ 
because $\sqrt{K(s,s)}$ and $\mu(s)$ have the same scaling as $x(s)$.
Similarly, other sampling frequencies are also zero shot transferable for our generalization measure by simply adjusting the SSM parameters.

\subsection{Generalization bound as an initialization scheme}\label{section: initialization scheme}

In this section, we design a scaling rule for the SSM parameters at initialization based on the generalization bound (\ref{thm: bound}).
This new initialization scheme improves the robustness of the initial output value scales on SSMs across different temporal patterns of the sequence data.

Our proposed initialization scheme is built on the HiPPO based initialization \citep{gu2023how}, which is a \textit{data independent} initialization method.
Specifically, the HiPPO framework initializes the hidden state matrices $A, B$ to produce orthogonal basis functions, and the matrix $C$ to be standard normal for training stability.
However, the argument for the training stability relies on the prerequisite that the input sequence is constant along the time index (\citet[Corollary 3.4]{gu2023how}),
which has some limitations in applicability as the long-range dependencies may lead to very different temporal patterns on the input sequence.
As the dashed lines in the left and the right part of Figure \ref{fig: toy result} show, the 
SSM output value scale  and the loss value scale under the HiPPO based initialization
vary much across different temporal dependencies, making the loss values inconsistent during training.
To address this issue, we follow the 
logic diagram in Figure \ref{fig: logic diagram} by adjusting the generalization complexity to be $\mathcal{O}(1)$.
Specifically, we extract the dominant term in the generalization bound (\ref{thm: bound}):
\begin{equation}\label{eq: proposed regularizer}
    \tau(\theta) := \left(\int_0^T \left| {\rho}_\theta(T-s)\right| \sqrt{K(s, s)}d s + \left|\int_0^T  {\rho}_\theta(T-s) \mu(s)d s\right|\right)^2.
\end{equation}
Notice that ${\rho}_\theta(s) = C e^{As} B$, if we rescale $C$ to $\xi C$ for some $\xi \in \mathbb{R}$, we have $\tau(\Tilde{\theta}) = \xi^2 \cdot \tau(\theta)$ for $\Tilde{\theta} = (\xi C, A, B)$.
This induces a new initialization scheme, i.e., once the parameters $\theta = (C, A, B)$ are initialized by the HiPPO method, 
we rescale $C$ to $\Tilde{C}$ such that  
\begin{equation}\label{eq: normalized C}
    \Tilde{C} = 
    \frac{1}{\sqrt{\tau(\theta)}} \cdot C.
\end{equation}
This rescaling method guarantees the SSM output value is bounded at initialization for \textit{any} stochastic process that satisfies Assumption \ref{assumption: as finite}, ensuring 
the robustness of the initial loss value scales on SSMs across different temporal structures.
We formalize the statement in Proposition \ref{prop: initialization}.
\begin{proposition}\label{prop: initialization}
    Consider a SSM $\int_0^T  {\rho}_\theta(T-s) x(s) d s$ with $\theta = (C, A, B)$, for any random process $x(s)$ satisfies Assumption 
    \ref{assumption: as finite},
    let $\Tilde{C}$ be given by the rescaling method (\ref{eq: normalized C}), then for $\Tilde{\theta} := (\Tilde{C}, A, B)$, 
    we have $\mathbb{E}_x \left[\left|\int_0^T  {\rho}_{\Tilde{\theta}}(T-s) x(s) d s\right|\right] \leq 
     \mathcal{O} (\sqrt{\log T})$.
     Here $\mathcal{O}$ hides a constant that only depends on $\sigma$ and $L$ as described in Assumption \ref{assumption: as finite}.
\end{proposition}
The proof is provided in Appendix \ref{appendix: proof of proposition 1}.
Proposition \ref{prop: initialization} shows that the expected SSM output value are bounded for \textit{any} stochastic processes that satisfies Assumption \ref{assumption: as finite}, even when the input sequence is not almost surely bounded.
This improves the robustness of the output value scales on SSMs 
in the sense that the scale of the output value does not depend on the variations of the temporal structures.
It is worth noting that
different from normalization methods such as min-max normalization and standardization, our method only changes the model parameters.
This is important because normalization on data numerical values in language tasks can lead to loss of crucial information. 
For example, mathematical expressions like ``$\max(1, 9) = 9$” have a contextual meaning where normalization may result in the loss of structured information that is essential to understand.

\begin{algorithm}[tb]
\caption{Training an $\ell$-layer SSM with the scheme (\ref{eq: normalized C})}
\label{alg: normalization}
\begin{algorithmic}[1]
\Require
training sequences with length $L$,
    model dimension $d$,
    projection matrix $C_i$ of $i$-th layer, 
    number of epochs $S$
\For{$s=0$ to $S-1$}
\If{$s = 0$}
\State
Sample a minibatch sequence from the training sequences
\For{$i=1$ to $\ell$}
\State
Compute mean $\mu_i \in \mathbb{R}^{L \times d}$, variance $K_i \in \mathbb{R}^{L \times d}$ of the $i$-th layer's inputs along the batch dimension
\Comment{Inputs of the $i$-th layer depend on model parameters of the first $i-1$ layers}
\State
Calculate the SSM kernel $k_i \in \mathbb{R}^{L \times d}$ by the model parameters of the $i$-th layer
\State
$\tau_i \gets \left[|k_i| * \sqrt{K_i} + |k_i * \mu_i|\right]_L \in \mathbb{R}^{d}$
\State
Averaging over the feature dimension: 
    $\tau_i \gets \|\tau_i\|_2^2 / d$
\State
Update: $C_i \gets {C_i}/ \sqrt{\tau_i}$
\EndFor
\EndIf
\State
Regular training procedure for the updated initialization
\EndFor
\end{algorithmic}
\end{algorithm}

\textbf{Implementation for high-dimensional, multi-layer SSMs.} 
In the practical training,  
the SSMs used for tasks such as image classification or language processing are usually deep and high dimensional ($d > 1$), while our initialization scheme (\ref{eq: normalized C}) is designed based on the one-dimensional shallow SSM.
To extend to high-dimensional SSMs, we empirically treat all features to be independent and calculate $\tau(\theta)$ 
by its average along the feature dimension.
For an $\ell$-layer SSM with the initial projection matrix $C_1,\ldots,C_\ell$ at each layer, we first calculate the complexity measure $\tau_1$ for the first layer and rescale $C_1$ by $C_1/\sqrt{\tau_1}$.
Then we calculate the complexity measure $\tau_2$ for the second layer by the updated input sequence of layer 2 and rescale $C_2$ by $C_2/\sqrt{\tau_2}$. 
We repeat this process until the last layer.
We describe the complete procedures in Algorithm \ref{alg: normalization}, 
where the $|\cdot|$ and $\sqrt{\cdot}$ in Line $7$ represent to element-wise absolute value and element-wise square root respectively.
$[\cdot]_L$ extracts the last position of an element obtained from the convolution.
The extension of our theory to the multi-layer case is an interesting direction, which we leave for future work.

\textbf{Skip connections and nonlinearities.}
There are several gaps between the theory and the methodologies in this paper.
The first one that the skip connection matrix $D$ is omitted in our defined model (\ref{eq: def SSM}).
This will not affect our generalization bound because we may express the explicit solution for (\ref{eq: def SSM}) as $y(t) = \int_0^t (\rho_\theta(s) + D \delta (s)) x(t-s) d s$ where $\delta(\cdot)$ is a delta function.
In that case, the SSM is still a convolution model but with a new kernel function $\rho_\theta(s) + D \delta (s)$.
However, the initialization scheme (\ref{eq: normalized C}) only adjusts $C$ and requires the kernel function to be linear in $C$.
Hence, (\ref{eq: normalized C}) may not work well when $D x(t)$ is much larger than $\int_0^t \rho_\theta(s)x(t-s)ds$. 
However, we can still derive a proper rescaling scheme for this case.
One straightforward way is that we first calculate $\tau(\theta)$ for a given initialization, and then rescale $C, D$ as $C \sqrt{\tau(\theta)}$ and $D / \sqrt{\tau(\theta)}$ respectively.
This reinitialization method guarantees that $\tau(\theta) = 1$ after rescaling.
The second gap is that our theory is for single-layer linear SSMs.
When nonlinearities are added, our generalization bound still works for single-layer SSMs if the nonlinearity does not affect the Hölder condition and the sub-Gaussian property (Assumption \ref{assumption: as finite}). 
For Lipschitz (also Hölder continuous) nonlinearities, there are some known examples (see Appendix \ref{appendix: lip for subgaussian}) where the sub-Gaussian condition still remains after the nonlinearity.

\subsection{Generalization bound as a regularization method}\label{section: regulaization method}

\begin{algorithm}[tb]
\caption{Training an $\ell$-layer SSM with the scheme (\ref{eq: regularized risk})}
\label{alg: regularization}
    \begin{algorithmic}[1]
\Require
training sequences with length $L$,
   model dimension $d$,
    initialization $\theta_0$,
    loss function $\mathcal{L}$,
    regularization factor $\lambda$,
    optimizer $\textsf{OPT}$,
    number of epochs $S$
\For{$s=0$ to $S-1$}
\State
Sample a minibatch from the training sequences
\State
Set total complexity $\tau = 0$
\For{$i=1$ to $\ell$}
\State
Compute mean $\mu_i \in \mathbb{R}^{L \times d}$ and variance $K_i \in \mathbb{R}^{L \times d}$ for inputs of the $i$-th layer along the batch dimension
\State
Calculate the SSM kernel $k_i \in \mathbb{R}^{L \times d}$ by the model parameters of the $i$-th layer
\State
$\tau_i \gets \left[|k_i| * \sqrt{K_i} + |k_i * \mu_i|\right]_L \in \mathbb{R}^{d}$
\State
Averaging over the feature dimension: 
    $\tau_i \gets \|\tau_i\|_2^2 / d$
\State
Add the complexity of the $i$-th layer: 
    $\tau \gets \tau + \tau_i$
\EndFor
\State
 Compute the training loss: $\mathcal{L} \gets \mathcal{L} + \lambda \cdot \tau$
 \State
 Parameters update: $\theta_{i+1} \gets \textsf{OPT}(\theta_i, \mathcal{L})$
\EndFor
\Ensure Updated model parameter $\theta_s$
\end{algorithmic}
\end{algorithm}

\begin{figure*}[ht]
    \centering
    \includegraphics[width=\textwidth]{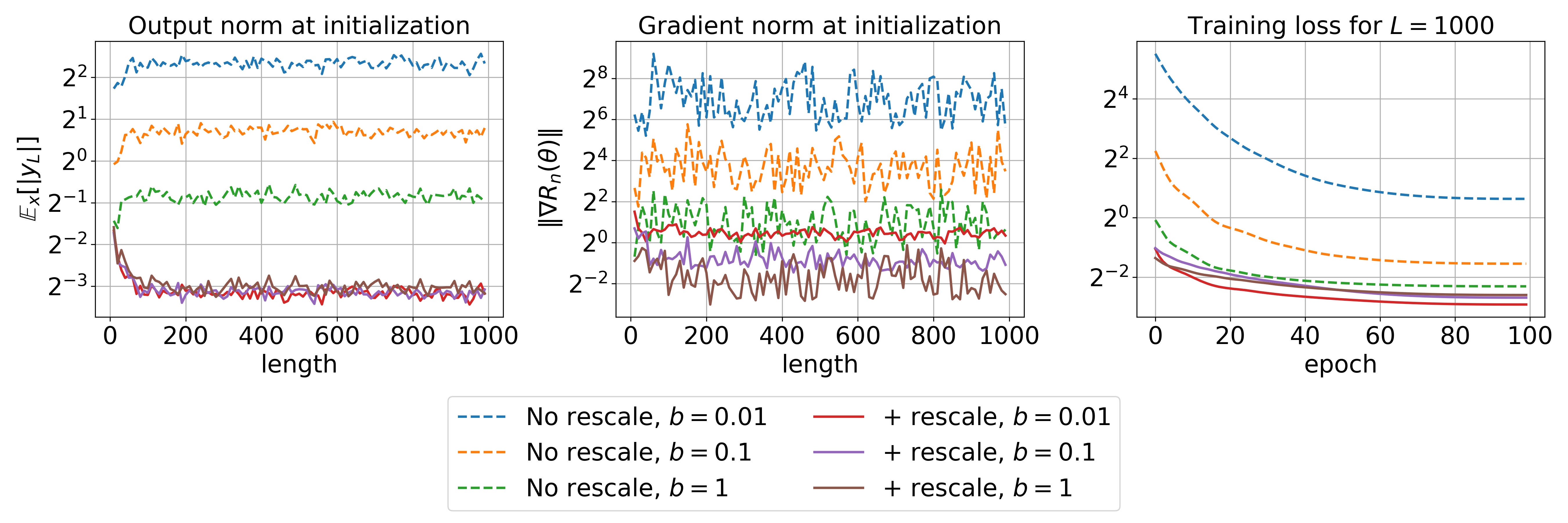}
    \caption{Effects of the initialization scheme (\ref{eq: normalized C}) on the model output scale, the gradient norm and the training loss under different temporal dependencies by varying the moment coefficient $b = 0.01, 0.1, 1$.
    (Left) The output $\mathbb{E}_x[|y_L|]$ at initialization w.r.t. the Gaussian white noise sequence $(x_1,\ldots,x_L)$ for length $L$ from $1$ to $1000$;
    (Middle) The gradient norm $\|\nabla R_n(\theta)\|$ at initialization w.r.t. the mean squared error (MSE) for varied sequence length;
    (Right) The training MSE curve for the Gaussian white noise with length $L = 1000$.}
    \label{fig: toy result}
\end{figure*}

In addition to its role as an initialization scheme, the generalization measure can also be regarded as a regularizer.
In this section, we utilize the bound (\ref{thm: bound}) to design a regularization method to improve the generalization performance, and simultaneously bring a little extra computational cost. 
For the generalization bound (\ref{thm: bound}),
we consider to use the dominant term (for large $T$) $\tau(\theta)$ defined in (\ref{eq: proposed regularizer})
as a regularizer.
Then, the new empirical risk with regularization is given by 
\begin{equation}\label{eq: regularized risk}
    \Tilde{R}_n(\theta) := 
    R_n(\theta) + 
    \lambda 
    \cdot \tau (\theta),
\end{equation}
where $\lambda \geq 0$ is the regularization coefficient.
When training multi-layer SSMs, we 
calculate the complexity $\tau(\theta)$ in (\ref{eq: regularized risk}) at each layer and add them together as a total regularization.
We describe the training procedures for one-layer SSMs in Algorithm \ref{alg: regularization}, 
where the notations follow Algorithm \ref{alg: normalization}.

\textbf{Computational cost analysis.}
From the training procedures in Algorithm \ref{alg: regularization}, we can see that the newly introduced training complexity mainly comes from the calculation for the convolution between the SSM kernel and the sequence statistics $(\mu, K)$.
Since the convolution can be conducted by the 
fast Fourier transform \citep{gu2022efficiently} with complexity $\mathcal{O}(bd L \log L)$ where $b$ is the batch size. Then the new complexity for Algorithm \ref{alg: regularization} becomes $\mathcal{O}((b+2)d L \log L)$, which is acceptable in the practical training.
We also include a concrete comparison of the running times in training real datasets to confirm this in Table \ref{table: lra normalization}.

\subsection{Experiments}\label{section: experiment}

\begin{table*}[t]
\begin{center}
\resizebox{\textwidth}{!}{%
\begin{tabular}{|c|ccc|ccc|ccc|}
\hline
& \multicolumn{3}{c|}{Training loss (MSE)} & \multicolumn{3}{c|}{Test loss (MSE)} & \multicolumn{3}{c|}{Generalization measure $\psi_{\Theta}^2/\sqrt{n}$} \\
\cline{2-10}
& $b=1$ & $b=0.1$ & $b=0.01$ & $b=1$ & $b=0.1$ & $b=0.01$ & $b=1$ & $b=0.1$ & $b=0.01$
\\
\hline
w/o (\ref{eq: normalized C}, \ref{eq: regularized risk}) & {0.15}$_{\pm 0.002}$ & ${0.67}_{\pm 0.09}$ & $2.50_{\pm 0.52}$ & $0.25_{\pm 0.01}$
& $1.01_{\pm 0.14}$ & $4.70_{\pm 0.77}$ & $0.93_{\pm 0.20}$ & $5.16_{\pm 0.84}$ & $46.23_{\pm 7.49}$
\\
\hline
w (\ref{eq: normalized C}) & $\textbf{0.11}_{\pm 0.006}$ & $\textbf{0.27}_{\pm 0.02}$ & $\textbf{0.20}_{\pm 0.01}$ & ${0.20}_{\pm 0.003}$ & $0.75_{\pm 0.05}$
& $1.06_{\pm 0.12}$ & $0.45_{\pm 0.03}$ & $2.36_{\pm 0.44}$ & $7.19_{\pm 1.60}$
\\
\hline
w (\ref{eq: regularized risk}) & $0.21_{\pm 0.008}$ & $0.97_{\pm 0.12}$ & $4.83_{\pm 0.52}$ & $0.22_{\pm 0.008}$
& $0.87_{\pm 0.07}$ & $3.59_{\pm 0.09}$ &
$0.55_{\pm 0.05}$ &
$2.76_{\pm 0.23}$ & $22.49_{\pm 0.78}$
\\
\hline
w (\ref{eq: normalized C}, \ref{eq: regularized risk}) & $0.15_{\pm 0.005}$ & $0.37_{\pm 0.04}$ & $0.35_{\pm 0.02}$ & $\textbf{0.18}_{\pm 0.004}$
& $\textbf{0.59}_{\pm 0.03}$ & $\textbf{0.60}_{\pm 0.01}$ &
$\textbf{0.23}_{\pm 0.01}$ &
$\textbf{0.46}_{\pm 0.12}$ & $\textbf{0.46}_{\pm 0.07}$
\\
\hline
\end{tabular}
}
\end{center}
\caption{Training loss, test loss and generalization measure $\psi_{\Theta}^2/\sqrt{n}$ on Gaussian white noise sequences with different coefficients $b$ after convergence.
By adding the initialization scheme (\ref{eq: normalized C}), SSMs achieve better optimization performance and are more robust on the final training loss value across different temporal dependencies. 
By adding the regularization term (\ref{eq: regularized risk}), SSMs
get better generalization performance (lower test loss).
}
\label{table: toy result}
\end{table*}

\begin{table*}[t]
\vspace{-2mm}
\begin{center}
\resizebox{\textwidth}{!}{%
\begin{tabular}{|cc|c|c|c|c|c|c|c|}
\hline
\multicolumn{2}{|c|}{}     & ListOps & Text & Retrieval & Image & Pathfinder & PathX & Average \\ 
\hline \hline
\multicolumn{1}{|c|}{\multirow{7}{*}{S4-Legs}}   & w/o (\ref{eq: normalized C}, \ref{eq: regularized risk}) &    {61.16}$_{\pm 0.32}$   &    {88.69}$_{\pm 0.07}$  &  91.21$_{\pm 0.17}$     &   87.41$_{\pm 0.14}$    &      95.89$_{\pm 0.10}$       &    96.97$_{\pm 0.31}$ & 86.89
\\ \cline{2-9} 
\multicolumn{1}{|c|}{}  & w (\ref{eq: normalized C})  &  60.79$_{\pm 0.26}$   &    88.58$_{\pm 0.21}$  &   {91.29$_{\pm 0.26}$}   & 87.67$_{\pm 0.29}$      &   95.79$_{\pm 0.31}$        &   95.99$_{\pm 0.18}$ & 86.69
\\ \cline{2-9} 
\multicolumn{1}{|c|}{}  & w  (\ref{eq: regularized risk}) & \textbf{61.63$_{\pm 0.10}$}   &  {88.80$_{\pm 0.27}$}    &   91.17$_{\pm 0.17}$   &     88.27$_{\pm 0.14}$  &     \textbf{96.02}$_{\pm 0.16}$     &   {97.18$_{\pm 0.20}$} & \textbf{87.18}
\\ \cline{2-9} 
\multicolumn{1}{|c|}{}                                         & w (\ref{eq: normalized C}, \ref{eq: regularized risk})  &    61.04$_{\pm 0.25}$     &   88.53$_{\pm 0.04}$   &    91.21$_{\pm 0.31}$    &   \textbf{88.63}$_{\pm 0.21}$    &       95.92$_{\pm 0.45}$    &    96.51$_{\pm 0.53}$ &  86.97
\\
\cline{2-9} 
\multicolumn{1}{|c|}{}  & Time / epoch, w/o (\ref{eq: normalized C}, \ref{eq: regularized risk})  & 5min 39s       &  3min 24s    &     17min 21s   &    1min 55s  &    3min 25s     &  67min 41s & 16min 34s  \\
\cline{2-9} 
\multicolumn{1}{|c|}{}  & Time / epoch, w (\ref{eq: regularized risk})  &   6min 03s     &  4min 03s   &   19min 19s     &  2min 08s    &  3min 50s      & 73min 10s & 18min 6s
\\
 \hline \hline
\multicolumn{1}{|c|}{\multirow{6}{*}{S4D-Legs}}  & w/o (\ref{eq: normalized C}, \ref{eq: regularized risk}) &     60.80$_{\pm 0.39}$    &   87.87$_{\pm 0.03}$ &   90.68$_{\pm 0.14}$     &   86.69$_{\pm 0.29}$    &    94.87$_{\pm 0.06}$       &   97.34$_{\pm 0.07}$ & 86.38
\\ \cline{2-9} 
\multicolumn{1}{|c|}{}   & w (\ref{eq: normalized C})  &  60.97$_{\pm 0.27}$      &   {87.83}$_{\pm 0.16}$   &   91.08$_{\pm 0.19}$   &      87.89$_{\pm 0.11}$ &     {94.72}$_{\pm 0.21}$     &   {95.86$_{\pm 0.66}$}   & 86.40
\\ \cline{2-9} 
\multicolumn{1}{|c|}{}                                         & w (\ref{eq: regularized risk})  & 61.32$_{\pm 0.43}$        &   88.02$_{\pm 0.06}$   &   91.10$_{\pm 0.11}$    &   87.98$_{\pm 0.09}$   &   \textbf{95.04}$_{\pm 0.07}$        &    \textbf{97.46}$_{\pm 0.15}$ & \textbf{86.82}
\\ \cline{2-9} 
\multicolumn{1}{|c|}{}   & w (\ref{eq: normalized C}, \ref{eq: regularized risk})  &     \textbf{61.48}$_{\pm 0.09}$    &     \textbf{88.19$_{\pm 0.42}$} &   \textbf{91.25$_{\pm 0.17}$}   &  \textbf{88.12}$_{\pm 0.25}$     &     {94.93}$_{\pm 0.30}$       &   95.63$_{\pm 0.48}$ & 86.60
\\
\cline{2-9} 
\multicolumn{1}{|c|}{}  & Time / epoch, w/o (\ref{eq: normalized C}, \ref{eq: regularized risk})  &  5min 10s      &   3min 07s   &   16min 37s   & 1min 42s   &   3min 02s  &  49min 39s & 13min 13s
\\ 
\cline{2-9} 
\multicolumn{1}{|c|}{}  & Time / epoch, w (\ref{eq: regularized risk})  &      5min 33s      &  3min 13s    &    18min 43s    &     1min 56s &   3min 28s     & 55min 33s &  14min 44s
\\
\hline
\end{tabular}
}
\end{center}
\caption{Test accuracy and running time (per epoch on A100 GPU) on the LRA benchmark under different settings for different models. 
Mean and standard error are reported based on 3 independent runs.}
\label{table: lra normalization}
\end{table*}

This section contains experiments to demonstrate the effectiveness of the proposed initialization scheme (\ref{eq: normalized C}) and the regularization method (\ref{eq: regularized risk}).
We use a synthetic dataset and the Long Range Arena (LRA) benchmark \citep{tay2021long} for numerical validations. 
To simplify the notation, we use w/o (\ref{eq: normalized C}, \ref{eq: regularized risk}), w (\ref{eq: normalized C}), w (\ref{eq: regularized risk}) and w (\ref{eq: normalized C}, \ref{eq: regularized risk}) to represent the baseline training without rescaling and regularization, training with rescaling, training with regularization and training with both rescaling and regularization respectively.

\textbf{A synthetic dataset.}
We consider a synthetic sequence dataset generated by a Gaussian white noise.
To more closely resemble real datasets, we generate training inputs by sampling data from non-centered Gaussian white noise with mean $\mu(s) = 1$ and covariance $K(s,t) = \frac{1}{|b| \sqrt{\pi}} e^{-((s-t)/b)^2}$, which is a stationary Gaussian process and satisfies Assumption \ref{assumption: as finite} (see Section \ref{section: generalization bound}).
Then we can get different temporal dependencies by varying the coefficient $b$, i.e., as the magnitude of $b$ decreasing, the temporal dependence of the corresponding Gaussian white noise decreases as well.
In particular, as $b \xrightarrow{} 0$, $\frac{1}{|b| \sqrt{\pi}} e^{-(x/b)^2}$ becomes a delta function $\delta(x)$, entailing a zero temporal dependence for the sequence data.

In the following experiment, we generate the sequence data by the Gaussian white noise with $b = [1, 0.1, 0.01]$.
For each input sequence $(x_1,\ldots,x_L)$, its corresponding label is obtained by $\sin(x_{[L/2]})$, i.e., the sine value of the time-lagged input.
We use the S4-Legs model \citep{gu2022efficiently} (that only contains the convolution layer) to train the sequence data.
More details about the experiment setup are provided in Appendix \ref{appendix: toy experiment}. 
In Figure \ref{fig: toy result}, we plot the model output $\mathbb{E}_x[|y_L|]$, the gradient norm $\|\nabla R_n(\theta)\|$ at initialization, and the training loss (w (\ref{eq: normalized C})) with different temporal patterns by varying the Gaussian white noise parameter $b$.
We see that the initialization scheme (\ref{eq: normalized C}) enhances the robustness of the output value scales (matches with Proposition \ref{prop: initialization}), gradient norm at initialization and also the training loss value across different temporal structures.
In Table \ref{table: toy result}, we report the training loss, test loss and the dominant generalization measure $\psi_{\Theta}^2/\sqrt{n}$ after convergence. 
By comparing the final training loss with and without (\ref{eq: normalized C}) in Table \ref{table: toy result} (w/o (\ref{eq: normalized C}, \ref{eq: regularized risk}) vs w (\ref{eq: normalized C}) and w (\ref{eq: regularized risk}) vs w (\ref{eq: normalized C}, \ref{eq: regularized risk})), 
we see that adding the rescaling scheme (\ref{eq: normalized C}) also improves the training performance and makes the final training error more robust on different temporal dependencies (by varying $b$).
For the regularization method (\ref{eq: regularized risk}), we compare the final test loss with and without (\ref{eq: regularized risk}) in Table \ref{table: toy result} (w/o (\ref{eq: normalized C}, \ref{eq: regularized risk}) vs w (\ref{eq: regularized risk}) and w (\ref{eq: normalized C}) vs w (\ref{eq: normalized C}, \ref{eq: regularized risk})).
We can see that the our regularization method improves the generalization performance.
Moreover, combining (\ref{eq: normalized C}) and (\ref{eq: regularized risk}), the model get the best test performance across various temporal structures of the sequence data.
The positive correlation between the generalization measure $\psi_\Theta/\sqrt{n}$ and the test loss across different $b$ indicates that our generalization bound is able to capture different temporal dependencies.

\textbf{LRA benchmark.}
For real datasets, we investigate the effects of the initialization scheme (\ref{eq: normalized C}) and the regularization method (\ref{eq: regularized risk}) on the LRA benchmark, which contains 6 tasks ranging from image classification to language processing.
We consider to train two base models: $6$-layer S4-Legs \citep{goel2022s} and $6$-layer S4D-Legs \citep{gu2022s4d}.
For the S4-Legs model, the hidden state matrix $A$ is a full matrix while for the S4D-Legs model, $A$ is a diagonal matrix. 
We follow the training rules as described in \citet{gu2023how}.  
When training with regularization (\ref{eq: regularized risk}), we vary the regularization coefficient $\lambda$ with $10^{-3}, 10^{-4}, 10^{-5}$ for ListOps, Text, Retrieval, Image and Pathfinder tasks.
For the most challenging task PathX, $\lambda$ is taken from $10^{-4}, 10^{-5}, 10^{-6}$.
We report the best test accuracy when training with regularization (\ref{eq: regularized risk}), and we include the exact running time for each epoch in Table \ref{table: lra normalization}.
Note that the reproduction of the baseline numbers (w/o (\ref{eq: normalized C}, \ref{eq: regularized risk})) is inconsistent with the results in \citep{gu2022s4d}.
This is because we do not use the same PyTorch version and CUDA version as suggested in the official codebase, which may lead to the performance difference. However, these slight differences do not affect the scientific conclusions we draw from this paper.

By comparing the best test accuracy for w/o (\ref{eq: normalized C}, \ref{eq: regularized risk}) vs w (\ref{eq: regularized risk}) and w (\ref{eq: normalized C}) vs w (\ref{eq: normalized C}, \ref{eq: regularized risk}) in Table \ref{table: lra normalization}, 
we see that adding the regularization (\ref{eq: regularized risk}) enhances the generalization performance (test accuracy) in almost all the tasks for both S4-Legs and S4D-Legs models.
When only adding the initialization scheme, by comparing w (\ref{eq: normalized C}) vs w/o (\ref{eq: normalized C}, \ref{eq: regularized risk}), the rescaling method becomes less effective compared to the synthetic case. 
This is because for the LRA benchmark, we follow the the original S4 paper \citep{gu2023how} to add the batch norm/layer norm to the model, which may potentially help to decrease the temporal dependencies of the data, and thus the rescaling method is not so much effective as in the synthetic case.
However, when combining the initialization scheme (\ref{eq: normalized C}) and the regularization (\ref{eq: regularized risk}), one can still get the best test performance in half of tasks, indicating that our proposed optimization designs help to improve the generalization performance.
We also compare the running time without or with the proposed optimization designs.
Since (\ref{eq: normalized C}) is conducted before training which will not introduce additional training complexity, we report the running time for w/o (\ref{eq: normalized C}, (\ref{eq: regularized risk})) and w (\ref{eq: regularized risk}) in Table \ref{table: lra normalization}.
The results show that the regularization brings a little extra computational cost, matching the computational cost analysis in Section \ref{section: regulaization method}.
We provide an ablation study for the regularization coefficient $\lambda$ in Appendix \ref{appendix: lra experiment}.
Results in Table \ref{table: vary lam on bidirectional S4} and Table \ref{table: vary lam on bidirectional S4D-Legs} show that the test accuracy is much more sensitive to $\lambda$ for the Pathfinder and PathX tasks compared to other tasks, which aligns with the findings of in \citet{gu2023how} that challenging tasks are more sensitive to the hyperparameters.
More details on the dataset description and the experiment setup are given in Appendix \ref{appendix: lra experiment}.
We include additional experiment results in Appendix \ref{appendix: additional experiment} for small S4-Legs and S4D-Legs with either smaller depth or smaller feature dimension.
We can see in Table \ref{table: additional lra normalization} that the improvements for small models are more significant  (e.g., nearly $2\%$ on the most challenging PathX tasks for S4-Legs and $>1\%$ on the average accuracy for S4D-Legs).
We also provide comparisons for different regularization schemes for both synthetic and real dataset.
One regularization method is filter norm regularization, i.e., we regularize the $\ell_2$ norm of the filter $\rho_\theta$, and another is weight decay on the hidden matrix $A$.
Experiment results and details are shown in Appendix \ref{appendix: comparison with different regularizers}.

\section{Discussions}\label{section: discussion}

In this work, we study the optimization and the generalization for SSMs. Specifically, we give a data-dependent generalization bound, revealing an effect of the temporal dependencies of the sequence data on the generalization. 
Based on the bound, we design two algorithms to improve the optimization and generalization for SSMs across different temporal patterns. 
The first is a new initialization scheme, by which we rescale the initialization such that the generalization measure is normalized.
This initialization scheme improves the robustness of SSMs on the output scales across various temporal dependencies.
The second is a new regularization method, which enhances the generalization performance in sequence modeling with only little extra computation cost.
However, our theory does not apply to multi-layer SSMs and we do not address the 
feature dependencies when calculating the generalization measure (\ref{eq: proposed regularizer}) for high-dimensional SSMs, but simply treat all the features independent.
It is interesting to understand the effects of depth and feature structures on optimization and generalization of SSMs, which we leave for future work.

\section{Acknowledgement}

This research is supported by the National Research Foundation, Singapore, under the NRF fellowship (project No. NRF-NRFF13-2021-0005).

\bibliography{main}

\newpage

\appendix

\section{Experiments details}

In this section, we provide more details for the experiments of the
synthetic dataset and the LRA benchmark in Section \ref{section: experiment}.

\subsection{The synthetic experiment}\label{appendix: toy experiment}

For the Gaussian white noise sequences, we generate $100$ i.i.d. sequences for training and $1000$ i.i.d. sequences for test.
The timescale for the discrete sequences is set to be $1$, i.e., to generate a Gaussian white noise sequence with length $L$, we sample from a multivariate normal distribution with mean $1$ and covariance matrix $K_{i,j} = h(i-j)$ for $i, j \in [1:L]$, where $h(t) = \frac{1}{|b| \sqrt{\pi}} e^{-(t/b)^2}$.
The model that we use is the one-layer S4 model that only contains the FFTConv (fast Fourier transform convolution) layer and without activation and the skip connection ($D = 0$) \citep{gu2022efficiently}.
The state space dimension for the FFTConv layer is $64$, other settings such as the discretization, the initialization and the parameterization follow the default settings in \citet{gu2023how}, i.e., we use 
the ZOH discretization,
the LegS initialization and the exponential parameterization for the hidden state matrix $A$. 

For the optimizer, we follow \citet{gu2023how} to set the optimizer by groups.
For the (ZOH) timescale $\Delta$, the hidden state matrices $A, B$, we use Adam optimizer with learning rate $0.001$, while for the matrix $C$, we use AdamW with learning rate $0.01$ and decay rate $0.01$.
For all the parameters, we use the cosine annealing schedule.
The batch size is set to be $100$ (full batch) and the training epochs is $100$.
The regularization coefficient $\lambda$ used for training with (\ref{eq: regularized risk}) is set to be $0.01$ across all the temporal patterns.

\subsection{LRA benchmark}\label{appendix: lra experiment}

\textbf{Datasets.}
The datasets in the LRA benchmark contain 
(1) ListOps \citep{nangia2018listops}, a dataset that is made up of a list of mathematical operations with answers;
(2) Text \citep{maas-EtAl:2011:ACL-HLT2011}, a movie review dataset collected from IMDB, which is used for sentiment analysis;
(3) Retrieval \citep{radev-etal-2009-acl}, a task of retrieving documents utilizing byte-level texts from the ACL Anthology Network.
(4)Image \citep{krizhevsky2009learning}, a sequential CIFAR10 dataset used for sequence classification; 
(5) Pathfinder \citep{linsley2018learning}, a task that requires a model to tell whether two points in an image are connected by a dashed path.
(6) PathX, a similar but more challenge task as Pathfinder with a higher image resolution increased from $32 \times 32$ to $128 \times 128$.

\textbf{Models.}
The models consist of S4-Legs and S4D-Legs.
Both models use the default Legs initialization. 
Discretization and model parameterization are set to be consistent with \citet{gu2023how}.
For the optimizer, we also follow the standard setup in \citet{gu2023how} that the hidden state matrices are trained in a relatively small learning rate with no weight decay, while other parameters are trained with AdamW with a larger learning rate.
Let $D, H, N$ denote the depth, feature dimension and hidden state space dimension respectively, we summarize the model hyperparameters for S4-Legs and S4D-Legs in Table \ref{table: bidirectional s4 model paramters} and Table \ref{table: bidirectional s4d model paramters} respectively.

\begin{table}[ht]
\begin{center}
\begin{tabular}{|c|c|c|c|c|c|c|c|c|}
\hline 
&   $D$ &  $H$ & $N$  &  { Dropout } &  { Learning rate } &  { Batch size } & 
 { Epochs } &  { Weight decay } \\
\hline 
 { ListOps } & 6 & 256 & 4  & 0 & 0.01 & 32 & 40 & 0.05  \\
 \hline
 { Text } & 6 & 256 & 4 &   0 & 0.01 & 16 & 32 & 0.05 \\
 \hline
 { Retrieval } & 6 & 256 & 4 &  0 & 0.01 & 64 & 20 & 0.05  \\
 \hline
 { Image } & 6 & 512 & 64 &  0.1 & 0.01 & 50 & 200 & 0.05 \\
 \hline
 { Pathfinder } & 6 & 256 & 64 & 0.0 & 0.004 & 64 & 200 & 0.05 \\
\hline
 { PathX } & 6 & 256 & 64 & 0.0 & 0.0005 & 16 & 50 & 0.05 \\
 \hline
\end{tabular}
\end{center}
\caption{List of the S4-Legs model hyperparameters for the LRA benchmark.}
\label{table: bidirectional s4 model paramters}
\end{table}

\begin{table}[ht]
\begin{center}
\begin{tabular}{|c|c|c|c|c|c|c|c|c|}
\hline 
&   $D$ &  $H$ & $N$  &  { Dropout } &  { Learning rate } &  { Batch size } & 
 { Epochs } &  { Weight decay } \\
\hline 
 { ListOps } & 6 & 256 & 4  & 0 & 0.01 & 32 & 40 & 0.05  \\
 \hline
 { Text } & 6 & 256 & 4 &   0 & 0.01 & 16 & 32 & 0.05 \\
 \hline
 { Retrieval } & 6 & 256 & 4 &  0 & 0.01 & 64 & 20 & 0.05  \\
 \hline
 { Image } & 6 & 512 & 64 &  0.1 & 0.01 & 50 & 200 & 0.05 \\
 \hline
 { Pathfinder } & 6 & 256 & 64 & 0.0 & 0.004 & 64 & 200 & 0.05 \\
\hline
 { PathX } & 6 & 256 & 64 & 0.0 & 0.0005 & 16 & 50 & 0.05 \\
 \hline
\end{tabular}
\end{center}
\caption{List of the S4D-Legs model hyperparameters for the LRA benchmark.}
\label{table: bidirectional s4d model paramters}
\end{table}

\begin{table}[ht]
\vspace{4mm}
\centering
\begin{subtable}[t]{.3\linewidth}
\centering
\begin{tabular}{|c|c|}
\hline
& ListOps \\
\hline
$\lambda = 0$ & {61.16}$_{\pm 0.32}$ \\
\hline
$\lambda = 10^{-5}$ & {61.36}$_{\pm 0.30}$ \\
\hline
$\lambda = 10^{-4}$ & {61.11}$_{\pm 0.10}$ \\
\hline
$\lambda = 10^{-3}$ & \textbf{61.63}$_{\pm 0.10}$ \\
\hline
\end{tabular}
\end{subtable}\hfill
\begin{subtable}[t]{.3\linewidth}
\centering
\begin{tabular}{|c|c|}
\hline
 & Text \\
\hline
$\lambda = 0$ & {88.69}$_{\pm 0.07}$ \\
\hline
$\lambda = 10^{-5}$ & \textbf{88.80}$_{\pm 0.27}$ \\
\hline
$\lambda = 10^{-4}$ & {88.66}$_{\pm 0.20}$ \\
\hline
$\lambda = 10^{-3}$ & {88.71}$_{\pm 0.12}$ \\
\hline
\end{tabular}
\end{subtable}\hfill
\vspace{1em}
\begin{subtable}[t]{.3\linewidth}
\centering
\begin{tabular}{|c|c|}
\hline
 & Retrieval \\
\hline
$\lambda = 0$ & \textbf{91.21}$_{\pm 0.17}$ \\
\hline
$\lambda = 10^{-5}$ & {91.17}$_{\pm 0.17}$ \\
\hline
$\lambda = 10^{-4}$ & {89.77}$_{\pm 2.28}$ \\
\hline
$\lambda = 10^{-3}$ & {88.25}$_{\pm 2.66}$ \\
\hline
\end{tabular}
\end{subtable}\hfill
\begin{subtable}[t]{.3\linewidth}
\centering
\begin{tabular}{|c|c|}
\hline
 & Image \\
\hline
$\lambda = 0$ & {87.41}$_{\pm 0.14}$ \\
\hline 
$\lambda = 10^{-5}$& {87.43}$_{\pm 0.33}$ \\
\hline 
$\lambda = 10^{-4}$& {87.45}$_{\pm 0.39}$ \\
\hline
$\lambda = 10^{-3}$& \textbf{88.27}$_{\pm 0.14}$ \\
\hline
\end{tabular}
\end{subtable}\hfill
\begin{subtable}[t]{.3\linewidth}
\centering
\begin{tabular}{|c|c|}
\hline
 & Pathfinder \\
\hline
$\lambda = 0$ & {95.89}$_{\pm 0.10}$ \\
\hline
$\lambda = 10^{-5}$ &  \textbf{96.02}$_{\pm 0.16}$ \\
\hline
$\lambda = 10^{-4}$ & {95.81}$_{\pm 0.33}$ \\
\hline
$\lambda = 10^{-3}$ & {89.06}$_{\pm 8.31}$ \\
\hline
\end{tabular}
\end{subtable}\hfill
\begin{subtable}[t]{.3\linewidth}
\centering
\begin{tabular}{|c|c|}
\hline
 & PathX \\
\hline
$\lambda = 0$ & {96.97}$_{\pm 0.31}$ \\
\hline
$\lambda = 10^{-6}$ & \textbf{97.18}$_{\pm 0.20}$ \\
\hline
$\lambda = 10^{-5}$ &  {97.16}$_{\pm 0.13}$ \\
\hline 
$\lambda = 10^{-4}$ & \XSolidBrush \\
\hline
\end{tabular}
\end{subtable}
\caption{Test accuracy for S4-Legs on LRA benchmark by varying the regularization coefficient  $\lambda$.}
\label{table: vary lam on bidirectional S4}
\end{table}

\begin{table}[h!]
\vspace{4mm}
\centering
\begin{subtable}[t]{.3\linewidth}
\centering
\begin{tabular}{|c|c|}
\hline
& ListOps \\
\hline
$\lambda = 0$ & 60.80$_{\pm 0.39}$ \\
\hline
$\lambda = 10^{-5}$ & 60.85$_{\pm 0.62}$ \\
\hline
$\lambda = 10^{-4}$ & 60.80$_{\pm 0.44}$ \\
\hline
$\lambda = 10^{-3}$ & \textbf{61.32}$_{\pm 0.43}$ \\
\hline
\end{tabular}
\end{subtable}\hfill
\begin{subtable}[t]{.3\linewidth}
\centering
\begin{tabular}{|c|c|}
\hline
 & Text \\
\hline
$\lambda = 0$ & {87.87}$_{\pm 0.03}$ \\
\hline
$\lambda = 10^{-5}$ & {87.64}$_{\pm 0.17}$ \\
\hline
$\lambda = 10^{-4}$ & {87.87}$_{\pm 0.36}$ \\
\hline
$\lambda = 10^{-3}$ & \textbf{88.02}$_{\pm 0.06}$ \\
\hline
\end{tabular}
\end{subtable}\hfill
\vspace{1em}
\begin{subtable}[t]{.3\linewidth}
\centering
\begin{tabular}{|c|c|}
\hline
 & Retrieval \\
\hline
$\lambda = 0$ & {90.68}$_{\pm 0.14}$ \\
\hline
$\lambda = 10^{-5}$ & {91.04}$_{\pm 0.13}$ \\
\hline
$\lambda = 10^{-4}$ & {90.95}$_{\pm 0.20}$ \\
\hline
$\lambda = \times 10^{-3}$ & \textbf{91.10}$_{\pm 0.11}$ \\
\hline
\end{tabular}
\end{subtable}\hfill
\begin{subtable}[t]{.3\linewidth}
\centering
\begin{tabular}{|c|c|}
\hline
 & Image \\
\hline
$\lambda = 0$ & {86.69}$_{\pm 0.29}$ \\
\hline 
$\lambda = 10^{-5}$& {86.91}$_{\pm 0.12}$ \\
\hline 
$\lambda = 10^{-4}$& {86.96}$_{\pm 0.22}$ \\
\hline
$\lambda = 10^{-3}$& \textbf{87.98}$_{\pm 0.09}$ \\
\hline
\end{tabular}
\end{subtable}\hfill
\begin{subtable}[t]{.3\linewidth}
\centering
\begin{tabular}{|c|c|}
\hline
 & Pathfinder \\
\hline
$\lambda = 0$ & {94.87}$_{\pm 0.06}$ \\
\hline
$\lambda = 10^{-5}$ &  \textbf{95.04}$_{\pm 0.07}$ \\
\hline
$\lambda = 10^{-4}$ & {94.38}$_{\pm 0.15}$ \\
\hline
$\lambda = 10^{-3}$ & {64.56}$_{\pm 19.94}$ \\
\hline
\end{tabular}
\end{subtable}\hfill
\begin{subtable}[t]{.3\linewidth}
\centering
\begin{tabular}{|c|c|}
\hline
 & PathX \\
\hline
$\lambda = 0$ & {97.34}$_{\pm 0.07}$ \\
\hline
$\lambda = 10^{-6}$ & {97.32}$_{\pm 0.14}$ \\
\hline
$\lambda = 10^{-5}$ &  \textbf{97.46}$_{\pm 0.15}$ \\
\hline 
$\lambda = 10^{-4}$ & \XSolidBrush \\
\hline
\end{tabular}
\end{subtable}
\caption{Test accuracy for S4D-Legs on LRA benchmark by varying the regularization coefficient  $\lambda$.}
\label{table: vary lam on bidirectional S4D-Legs}
\end{table}

\begin{table}[h!]
\begin{center}
\begin{tabular}{|c|c|c|c|c|c|c|c|c|}
\hline 
&   $D$ &  $H$ & $N$  &  { Dropout } &  { Learning rate } &  { Batch size } & 
 { Epochs } &  { Weight decay } \\
\hline 
 { ListOps } & 4 & 128 & 64  & 0 & 0.01 & 50 & 40 & 0.05  \\
 \hline
 { Text } & 4 & 128 & 64 &   0 & 0.01 & 50 & 50 & 0.0 \\
 \hline
 { Retrieval } & 4 & 96 & 4 &  0 & 0.01 & 64 & 20 & 0.05  \\
 \hline
 { Image } & 4 & 128 & 64 &  0.1 & 0.01 & 50 & 100 & 0.05 \\
 \hline
 { Pathfinder } & 6 & 128 & 64 & 0.0 & 0.004 & 64 & 40 & 0.01 \\
\hline
 { PathX } & 4 & 96 & 64 & 0.0 & 0.0005 & 64 & 50 & 0.05 \\
 \hline
\end{tabular}
\end{center}
\caption{List of the small S4-Legs model hyperparameters for the LRA benchmark.}
\label{table: bidirectional small s4 model paramters}
\end{table}

\begin{table}[h!]
\begin{center}
\begin{tabular}{|c|c|c|c|c|c|c|c|c|}
\hline 
&   $D$ &  $H$ & $N$  &  { Dropout } &  { Learning rate } &  { Batch size } & 
 { Epochs } &  { Weight decay } \\
\hline 
 { ListOps } & 4 & 128 & 64  & 0 & 0.01 & 50 & 40 & 0.05  \\
 \hline
 { Text } & 4 & 128 & 64 &   0 & 0.01 & 50 & 50 & 0.0 \\
 \hline
 { Retrieval } & 4 & 96 & 4 &  0 & 0.01 & 64 & 20 & 0.05  \\
 \hline
 { Image } & 4 & 128 & 64 &  0.1 & 0.01 & 50 & 100 & 0.05 \\
 \hline
 { Pathfinder } & 6 & 128 & 64 & 0.0 & 0.004 & 64 & 40 & 0.01 \\
\hline
 { PathX } & 4 & 96 & 64 & 0.0 & 0.0005 & 64 & 50 & 0.05 \\
 \hline
\end{tabular}
\end{center}
\caption{List of the small S4D-Legs model hyperparameters for the LRA benchmark.}
\label{table: bidirectional small s4d model paramters}
\end{table}

\begin{table*}[h!]
\begin{center}
\resizebox{\textwidth}{!}{%
\begin{tabular}{|cc|c|c|c|c|c|c|c|}
\hline
\multicolumn{2}{|c|}{}     & ListOps & Text & Retrieval & Image & Pathfinder & PathX & Average \\ 
\hline \hline
\multicolumn{1}{|c|}{\multirow{7}{*}{S4-Legs}}   & w/o (\ref{eq: normalized C}, \ref{eq: regularized risk}) &    {55.38}$_{\pm 0.76}$   &    {84.72}$_{\pm 0.40}$  &  85.75$_{\pm 0.46}$     &   82.07$_{\pm 0.11}$    &      89.36$_{\pm 0.38}$       &    88.75$_{\pm 0.62}$ & 81.01
\\ \cline{2-9} 
\multicolumn{1}{|c|}{}  & w (\ref{eq: normalized C})  &  53.72$_{\pm 1.59}$   &    85.21$_{\pm 0.21}$  &   {84.47$_{\pm 1.50}$}   & 83.71$_{\pm 0.21}$      &   89.16$_{\pm 1.38}$        &   88.96$_{\pm 1.62}$ & 80.87
\\ \cline{2-9} 
\multicolumn{1}{|c|}{}  & w  (\ref{eq: regularized risk}) & \textbf{55.43$_{\pm 1.55}$}   &  {85.12$_{\pm 0.34}$}    &   83.30$_{\pm 1,75}$   &     83.86$_{\pm 0.25}$  &     \textbf{89.39}$_{\pm 0.34}$     &   \textbf{90.70$_{\pm 0.61}$} & {81.30}
\\ \cline{2-9} 
\multicolumn{1}{|c|}{}                                         & w (\ref{eq: normalized C}, \ref{eq: regularized risk})  &    54.97$_{\pm 0.30}$     &   \textbf{85.27}$_{\pm 0.21}$   &    \textbf{85.82}$_{\pm 0.42}$    &   \textbf{84.74}$_{\pm 0.18}$    &       88.64$_{\pm 0.36}$    &    90.19$_{\pm 0.90}$ &  \textbf{81.61}
\\
\cline{2-9} 
\multicolumn{1}{|c|}{}  & Time / epoch, w/o (\ref{eq: normalized C}, \ref{eq: regularized risk})  & 2min 06s       &  50s    &     5min 57s   &    33s  &    2min 13s     &  10min 33s & 3min 42s  \\
\cline{2-9} 
\multicolumn{1}{|c|}{}  & Time / epoch, w (\ref{eq: regularized risk})  &   2min 18s     &  52s   &   6min 28s     &  37s    &  2min 31s      & 11min 46s & 4min 6s
\\
 \hline \hline
\multicolumn{1}{|c|}{\multirow{6}{*}{S4D-Legs}}  & w/o (\ref{eq: normalized C}, \ref{eq: regularized risk}) &     55.17$_{\pm 0.20}$    &   83.60$_{\pm 0.09}$ &   89.12$_{\pm 0.14}$     &   81.07$_{\pm 0.39}$    &    87.28$_{\pm 0.47}$       &   89.91$_{\pm 0.53}$ & 81.03
\\ \cline{2-9} 
\multicolumn{1}{|c|}{}   & w (\ref{eq: normalized C})  &  55.80$_{\pm 0.11}$      &   {85.30}$_{\pm 0.10}$   &   89.32$_{\pm 0.17}$   &      82.35$_{\pm 0.56}$ &     {88.00}$_{\pm 0.82}$     &   {90.15$_{\pm 0.86}$}   & 81.82
\\ \cline{2-9} 
\multicolumn{1}{|c|}{}                                         & w (\ref{eq: regularized risk})  & \textbf{56.45}$_{\pm 0.33}$        &   84.86$_{\pm 0.38}$   &   89.21$_{\pm 0.09}$    &   82.39$_{\pm 0.18}$   &   {87.86}$_{\pm 0.31}$        &    \textbf{90.95}$_{\pm 0.21}$ & {81.95}
\\ \cline{2-9} 
\multicolumn{1}{|c|}{}   & w (\ref{eq: normalized C}, \ref{eq: regularized risk})  &     {55.82}$_{\pm 0.66}$    &     \textbf{85.50$_{\pm 0.06}$} &   \textbf{89.34$_{\pm 0.04}$}   &  \textbf{83.79}$_{\pm 0.29}$     &     \textbf{88.53}$_{\pm 0.69}$       &   90.51$_{\pm 1.01}$ & \textbf{82.25}
\\
\cline{2-9} 
\multicolumn{1}{|c|}{}  & Time / epoch, w/o (\ref{eq: normalized C}, \ref{eq: regularized risk})  &  1min 53s      &   47s   &   5min 40s   & 29s   &   2min  &  9min 52s & 3min 27s
\\ 
\cline{2-9} 
\multicolumn{1}{|c|}{}  & Time / epoch, w (\ref{eq: regularized risk})  &      2min 11s      &  48s    &    6min 15s    &     34s &   2min 16s     & 11min 05s &  3min 52s
\\
\hline
\end{tabular}
}
\end{center}
\caption{Test accuracy and running time (per epoch on A100 GPU) on the LRA benchmark under different settings for small S4-Legs and S4D-Legs. 
Mean and standard error are reported based on 3 independent runs.}
\label{table: additional lra normalization}
\end{table*}

\begin{table}[h!]
\vspace{4mm}
\centering
\begin{subtable}[t]{.3\linewidth}
\centering
\begin{tabular}{|c|c|}
\hline
& ListOps \\
\hline
$\lambda = 0$ & {55.38}$_{\pm 0.76}$ \\
\hline
$\lambda = 10^{-5}$ & {55.32}$_{\pm 1.03}$ \\
\hline
$\lambda = 10^{-4}$ & \textbf{55.43}$_{\pm 1.55}$ \\
\hline
$\lambda = 10^{-3}$ & {55.33}$_{\pm 0.44}$ \\
\hline
\end{tabular}
\end{subtable}\hfill
\begin{subtable}[t]{.3\linewidth}
\centering
\begin{tabular}{|c|c|}
\hline
 & Text \\
\hline
$\lambda = 0$ & {84.72}$_{\pm 0.40}$ \\
\hline
$\lambda = 10^{-5}$ & {84.74}$_{\pm 0.21}$ \\
\hline
$\lambda = 10^{-4}$ & {84.62}$_{\pm 0.18}$ \\
\hline
$\lambda = 10^{-3}$ & \textbf{85.12}$_{\pm 0.34}$ \\
\hline
\end{tabular}
\end{subtable}\hfill
\vspace{1em}
\begin{subtable}[t]{.3\linewidth}
\centering
\begin{tabular}{|c|c|}
\hline
 & Retrieval \\
\hline
$\lambda = 0$ & \textbf{85.75}$_{\pm 0.46}$ \\
\hline
$\lambda = 10^{-5}$ & {83.30}$_{\pm 1.75}$ \\
\hline
$\lambda = 10^{-4}$ & {82.71}$_{\pm 1.18}$ \\
\hline
$\lambda = 10^{-3}$ & {82.09}$_{\pm 0.41}$ \\
\hline
\end{tabular}
\end{subtable}\hfill
\begin{subtable}[t]{.3\linewidth}
\centering
\begin{tabular}{|c|c|}
\hline
 & Image \\
\hline
$\lambda = 0$ & {82.07}$_{\pm 0.11}$ \\
\hline 
$\lambda = 10^{-5}$& {82.80}$_{\pm 0.32}$ \\
\hline 
$\lambda = 10^{-4}$& {82.98}$_{\pm 0.15}$ \\
\hline
$\lambda = 10^{-3}$& \textbf{83.86}$_{\pm 0.25}$ \\
\hline
\end{tabular}
\end{subtable}\hfill
\begin{subtable}[t]{.3\linewidth}
\centering
\begin{tabular}{|c|c|}
\hline
 & Pathfinder \\
\hline
$\lambda = 0$ & {89.36}$_{\pm 0.38}$ \\
\hline
$\lambda = 10^{-5}$ &  \textbf{89.39}$_{\pm 0.34}$ \\
\hline
$\lambda = 10^{-4}$ & {89.20}$_{\pm 0.19}$ \\
\hline
$\lambda = 10^{-3}$ & {50.54}$_{\pm 0.01}$ \\
\hline
\end{tabular}
\end{subtable}\hfill
\begin{subtable}[t]{.3\linewidth}
\centering
\begin{tabular}{|c|c|}
\hline
 & PathX \\
\hline
$\lambda = 0$ & {88.75}$_{\pm 0.62}$ \\
\hline
$\lambda = 10^{-6}$ & {88.51}$_{\pm 0.70}$ \\
\hline
$\lambda = 10^{-5}$ &  {89.71}$_{\pm 0.40}$ \\
\hline 
$\lambda = 10^{-4}$ & \textbf{90.70}$_{\pm 0.61}$ \\
\hline
\end{tabular}
\end{subtable}
\caption{Test accuracy for small S4-Legs on LRA benchmark by varying the regularization coefficient  $\lambda$.}
\label{table: vary lam on small bidirectional S4}
\end{table}

\begin{table}[h!]
\vspace{4mm}
\centering
\begin{subtable}[t]{.3\linewidth}
\centering
\begin{tabular}{|c|c|}
\hline
& ListOps \\
\hline
$\lambda = 0$ & 55.17$_{\pm 0.20}$ \\
\hline
$\lambda = 10^{-5}$ & \textbf{56.45}$_{\pm 0.33}$ \\
\hline
$\lambda = 10^{-4}$ & 56.03$_{\pm 1.36}$ \\
\hline
$\lambda = 10^{-3}$ & {55.48}$_{\pm 0.50}$ \\
\hline
\end{tabular}
\end{subtable}\hfill
\begin{subtable}[t]{.3\linewidth}
\centering
\begin{tabular}{|c|c|}
\hline
 & Text \\
\hline
$\lambda = 0$ & {83.60}$_{\pm 0.09}$ \\
\hline
$\lambda = 10^{-5}$ & {84.13}$_{\pm 0.48}$ \\
\hline
$\lambda = 10^{-4}$ & {84.48}$_{\pm 0.20}$ \\
\hline
$\lambda = 10^{-3}$ & \textbf{84.86}$_{\pm 0.38}$ \\
\hline
\end{tabular}
\end{subtable}\hfill
\vspace{1em}
\begin{subtable}[t]{.3\linewidth}
\centering
\begin{tabular}{|c|c|}
\hline
 & Retrieval \\
\hline
$\lambda = 0$ & {89.12}$_{\pm 0.14}$ \\
\hline
$\lambda = 10^{-5}$ & \textbf{89.21}$_{\pm 0.09}$ \\
\hline
$\lambda = 10^{-4}$ & {89.18}$_{\pm 0.11}$ \\
\hline
$\lambda = \times 10^{-3}$ & {88.97}$_{\pm 0.07}$ \\
\hline
\end{tabular}
\end{subtable}\hfill
\begin{subtable}[t]{.3\linewidth}
\centering
\begin{tabular}{|c|c|}
\hline
 & Image \\
\hline
$\lambda = 0$ & {81.07}$_{\pm 0.39}$ \\
\hline 
$\lambda = 10^{-5}$& {81.39}$_{\pm 0.35}$ \\
\hline 
$\lambda = 10^{-4}$& {81.71}$_{\pm 0.39}$ \\
\hline
$\lambda = 10^{-3}$& \textbf{82.39}$_{\pm 0.18}$ \\
\hline
\end{tabular}
\end{subtable}\hfill
\begin{subtable}[t]{.3\linewidth}
\centering
\begin{tabular}{|c|c|}
\hline
 & Pathfinder \\
\hline
$\lambda = 0$ & {87.28}$_{\pm 0.47}$ \\
\hline
$\lambda = 10^{-5}$ &  \textbf{87.86}$_{\pm 0.31}$ \\
\hline
$\lambda = 10^{-4}$ & {50.14}$_{\pm 0.57}$ \\
\hline
$\lambda = 10^{-3}$ & {50.54}$_{\pm 0.00}$ \\
\hline
\end{tabular}
\end{subtable}\hfill
\begin{subtable}[t]{.3\linewidth}
\centering
\begin{tabular}{|c|c|}
\hline
 & PathX \\
\hline
$\lambda = 0$ & {89.91}$_{\pm 0.53}$ \\
\hline
$\lambda = 10^{-6}$ & {89.79}$_{\pm 0.65}$ \\
\hline
$\lambda = 10^{-5}$ &  \textbf{90.95}$_{\pm 0.21}$ \\
\hline 
$\lambda = 10^{-4}$ & {86.32}$_{\pm 1.53}$ \\
\hline
\end{tabular}
\end{subtable}
\caption{Test accuracy for small S4D-Legs on LRA benchmark by varying the regularization coefficient  $\lambda$.}
\label{table: vary lam on small bidirectional S4D-Legs}
\end{table}

\begin{table*}[t]
\begin{center}
\begin{tabular}{|c|cccc|}
\hline
& \multicolumn{4}{c|}{Test loss (MSE)} \\
\cline{2-5}
& w/o (\ref{eq: normalized C}, \ref{eq: regularized risk}) & w (\ref{eq: regularized risk}) & Weight decay on $A$ & Filter norm regularization 
\\
\hline
$b=1$ & {0.25}$_{\pm 0.01}$ & $\textbf{0.22}_{\pm 0.008}$ & $0.24_{\pm 0.004}$ & $\textbf{0.22}_{\pm 0.007}$
\\
\hline
$b=0.1$ & 1.01$_{\pm 0.14}$ & $\textbf{0.87}_{\pm 0.07}$ & $0.97_{\pm 0.07}$ & $0.96_{\pm 0.12}$ 
\\
\hline
$b=0.01$ & $4.70_{\pm 0.77}$ & $\textbf{3.59}_{\pm 0.09}$ & $4.23_{\pm 0.23}$ & $4.61_{\pm 0.73}$
\\
\hline
\end{tabular}
\end{center}
\caption{Test loss for different regularization methods on synthetic data after convergence.}
\label{table: synthetic regularizer compare}
\end{table*}

\begin{table*}[t]
\centering
\begin{tabular}{|c|c|c|c|c|c|c|c|}
\hline
S4-Legs & ListOps & Text & Retrieval & Image & Pathfinder & PathX & Avg \\ \hline
w/o (\ref{eq: normalized C}, \ref{eq: regularized risk})       & 61.16$_{\pm 0.32}$ & 88.69$_{\pm 0.07}$ & 91.21$_{\pm 0.17}$ & 87.41$_{\pm 0.14}$ & 95.89$_{\pm 0.10}$ & 96.97$_{\pm 0.31}$ & 86.89 \\ \hline
w (\ref{eq: regularized risk})           & \textbf{61.63}$_{\pm 0.10}$ & 88.80$_{\pm 0.27}$ & 91.17$_{\pm 0.17}$ & \textbf{88.27}$_{\pm 0.14}$ & \textbf{96.02}$_{\pm 0.16}$ & 97.18$_{\pm 0.20}$ & \textbf{87.18} \\ \hline
Weight decay for $A$ & 49.90$_{\pm 0.67}$ & 86.58$_{\pm 0.91}$ & 91.21$_{\pm 0.17}$ & 87.65$_{\pm 0.16}$ & 96.00$_{\pm 0.09}$ & \textbf{97.22}$_{\pm 0.05}$ & 84.76 \\ \hline
Filter norm regularization & 61.53$_{\pm 0.39}$ & \textbf{88.88}$_{\pm 0.13}$ & \textbf{91.44}$_{\pm 0.08}$ & 87.70$_{\pm 0.20}$ & 95.83$_{\pm 0.14}$ & 97.16$_{\pm 0.16}$ & 87.09 \\ \hline
\end{tabular}
\caption{Test accuracy for different regularization methods on the LRA benchmark for S4-Legs.}
\label{table: LRA regularizer compare}
\end{table*}

\textbf{Ablation studies on $\lambda$.}
When training with the regularization method (\ref{eq: regularized risk}), we vary the regularization coefficient $\lambda$ for different magnitudes ranging from $10^{-6}$ to $10^{-3}$ when the model performs best on the validation set.
In Table \ref{table: vary lam on bidirectional S4} and Table \ref{table: vary lam on bidirectional S4D-Legs}, we report the test accuracy on the LRA benchmark with different $\lambda$ for the S4-Legs and S4D-Legs model respectively.
From the results in Table \ref{table: vary lam on bidirectional S4} and Table \ref{table: vary lam on bidirectional S4D-Legs},
we find that for both models, 
adding the regularization helps the generalization performance (test accuracy) for all the tasks except for the Retrieval task trained by the S4-Legs model. 
In particular, 
the test accuracy is much more sensitive to the regularization coefficient $\lambda$ for the Pathfinder and PathX tasks compared to other tasks.
For example, the variance of the test accuracy for the Pathfinder task is very high when $\lambda = 0.001$.
For the PathX task, both the S4-Legs and the S4D-Legs model can not even learn the dataset when $\lambda = 0.0001$.
The high sensitivity of the model in the hyperparameter aligns with the numerical findings in \citet{gu2023how}.

\subsection{Additional experiment results for small SSMs}\label{appendix: additional experiment}

In this section, we include more experiment results for smaller size of S4-Legs and S4D-Legs on the LRA benchmark.
The best test accuracy results and the running time for the small models are reported in Table \ref{table: additional lra normalization}.
The details for the model size and hyperparmeters are provided in Table \ref{table: bidirectional small s4 model paramters} and Table \ref{table: bidirectional small s4d model paramters}, where the notations follow from Table \ref{table: bidirectional s4 model paramters}.
The ablation studies on the regularization coefficient $\lambda$ (without the initialization scheme (\ref{eq: normalized C})) for the small S4-Legs and S4D-Legs are given in Table \ref{table: vary lam on small bidirectional S4} and Table \ref{table: vary lam on small bidirectional S4D-Legs}.

From Table \ref{table: additional lra normalization}, 
by comparing the test performance for w/o (\ref{eq: normalized C}, \ref{eq: regularized risk}) vs w (\ref{eq: regularized risk}) and w(\ref{eq: normalized C}) vs w (\ref{eq: normalized C}, \ref{eq: regularized risk}),
we can see that the regularization scheme (\ref{eq: regularized risk}) helps to improve the test performance for all the tasks except the Retrieval task for S4-Legs.
This is also verified in the ablation studies of the regularization coefficient $\lambda$, as shown in Table \ref{table: vary lam on small bidirectional S4} and Table \ref{table: vary lam on small bidirectional S4D-Legs}.
Combining the initialization scheme (\ref{eq: normalized C}) and the regularization method (\ref{eq: regularized risk}), more than half of the tasks can achieve the best test accuracy.
For both S4-Legs and S4D-Legs, integrating the two methods (\ref{eq: normalized C}) and (\ref{eq: regularized risk}) induces the best average test accuracy across all the $6$ tasks in the LRA benchmark. 
Therefore, our methods also work for small size of SSMs with a little extra computation cost.

\subsection{Comparisons with different regularization schemes}\label{appendix: comparison with different regularizers}

In this section, we add two additional regularization schemes for comparison.

\begin{enumerate}
    \item Filter norm regularization. 
    We regularize the $\ell_2$ norm of the filter $\rho_\theta$, i.e., when calculating the regularization measure $\tau(\theta)$, we simply take $\mu(s) = 0$ and $K(s,s)=1$ to ignore the effects of the temporal structure of the data.
    \item 
    Weight decay on the hidden matrix $A$. 
    In the original S4(D) papers \cite{gu2022efficiently, gu2022s4d, gu2023how}, the default training methods do not apply weight decay to the hidden matrix $A$, and there is no known ablation study on the effect of weight decay on $A$. Here we add weight decay to compare with the proposed regularization schemes.
\end{enumerate}

For synthetic task, we follow the experiment settings in the main paper.
The filter norm regularization results are obtained by following the same training settings in the paper. 
The weight decay results are chosen from the best weight decay coefficient from $10^{-3}, 10^{-2}, 10^{-1}, 10^0, 10^1$.
We report the test loss in Table \ref{table: synthetic regularizer compare}.
For the LRA benchmark, we also follow the same training setup in the paper to compare the performance of different regularization schemes on the S4-Legs model. 
The test accuracy for each task is shown in Table \ref{table: LRA regularizer compare}.
From the synthetic results, we see that our regularization scheme can achieve the best performance compared to the other regularization schemes across different temporal structures. 
For the LRA benchmark, the proposed regularization scheme also achieves the best performance on the average accuracy across different tasks. 
In particular, for the ListOps task, weight decay performs much worse than the other regularization methods.

\section{Proof for the linear regression result in Section \ref{section: motivation}.}
\label{appendix, proof motivation}

In this section, we give the proof for the generalization bound (\ref{eq: bound for linear regression}).
The proof is based on the following uniform-convergence generalization bound in \citet{Mohri}.
\begin{lemma}\label{mohri}
Consider a family of functions $\mathcal{F}$ mapping from $\mathcal{Z}$ to $[a, b]$. Let $\mathcal{D}$ denote the distribution according to which samples are drawn. Then for any $\delta > 0$, with probability at least $1 - \delta$ over the draw of an i.i.d. sample $S = \left\{z_1, \dots, z_n\right\}$, the following holds for all $f \in \mathcal{F}$:
\begin{equation*}
     \mathbb{E}_{z \sim \mathcal{D}} \left[f (z)\right] - \frac{1}{n} \sum_{i=1}^n f(z_i) \leq 2 \mathcal{R}_S (\mathcal{F}) + 3 (b-a) \sqrt{\frac{\log (2 / \delta)}{2 n}},
\end{equation*}
where $\mathcal{R}_S (\mathcal{F})$ is the empirical Rademacher complexity with respect to the sample $S$, defined as:
$\mathcal{R}_S (\mathcal{F}) = \mathbb{E}_{\sigma} \left[\sup_{f \in \mathcal{F}} \frac{1}{n} \sum_{i=1}^n \sigma_i f(z_i)\right]$.
$\left\{\sigma_i\right\}_{i=1}^n$ are i.i.d. random variables drawn from $U\{-1, 1\}$ with $P(\sigma_i=1) = P(\sigma_i=-1) = 0.5$.
\end{lemma}

And the Talagrand’s contraction lemma \citet{ledoux2013probability}.
\begin{lemma}\label{lemma: contraction}
    Let $H$ be a hypothesis set of functions mapping $\mathcal{X}$ to $\mathbb{R}$ and $\Psi_1,\ldots,\Psi_m$, $\mu$-Lipschitz functions for some $\mu>0$. 
    Then, for any sample $S$ of $m$ points $x_1,...,x_m \in \mathcal{X}$, the following inequality holds
    \begin{equation*}
        \frac{1}{m} {\mathbb{E}_\sigma}\left[\sup _{h \in H} \sum_{i=1}^m \sigma_i\left(\Psi_i \circ h\right)\left(x_i\right)\right] \leq \frac{\mu}{m} {\mathbb{E}_\sigma}\left[\sup _{h \in H} \sum_{i=1}^m \sigma_i h\left(x_i\right)\right]
    \end{equation*}
\end{lemma}

Now we begin our proof:
\begin{proof}
    First, notice for any $i \in [1:n]$ and $\theta \in \Theta$, we have 
    \begin{equation*}
        (\theta^\top x_i - y_i)^2 \leq 2 (\theta^\top x_i)^2 + 2y_i^2 \leq 2 r^2 R^2 +2
    \end{equation*}
    Second, note that $(\theta^\top x_i - y_i)^2$ is $2 \sup_{\theta \in \Theta, i \in [1:n]}|\theta^\top x_i - y_i|$-Lipschitz (the maximum gradient norm) with respect to $\theta^\top x_i - y_i$, and we can bound the Lipschitz constant as 
    \begin{equation*}
        2 \sup_{\theta \in \Theta, i \in [1:n]}|\theta^\top x_i - y_i| \leq 2 r R + 2
    \end{equation*}
    Then by Lemma \ref{lemma: contraction}, the Rademacher complexity for the linear model is bounded as 
    \begin{align*}
        \mathcal{R}_S(\mathcal{F}) & = \frac{1}{n} \mathbb{E}_\sigma \left[\sup_{\|\theta\|_2 \leq R} \sum_{i=1}^n \sigma_i (\theta^\top x_i - y_i)^2\right] \\
        & \leq \frac{2 r R + 2}{n} 
        \mathbb{E}_\sigma \left[\sup_{\|\theta\|_2 \leq R} \sum_{i=1}^n \sigma_i (\theta^\top x_i - y_i)\right] \\
        & = \frac{2 r R + 2}{n} 
        \mathbb{E}_\sigma \left[\sup_{\|\theta\|_2 \leq R} \sum_{i=1}^n \sigma_i \theta^\top x_i\right] \\
        & \leq 
        \frac{2 R( r R + 1)}{n} 
        \mathbb{E}_\sigma \left\|\sum_{i=1}^n \sigma_i x_i\right\| \\
        & \leq 
        \frac{2 R( r R + 1)}{n}
        \sqrt{\mathbb{E}_\sigma \left\|\sum_{i=1}^n \sigma_i x_i\right\|^2} \\
        & = \frac{2 R( r R + 1)}{n} \sqrt{\sum_{i=1}^n \|x_i\|^2} \\
        & \leq \frac{2 rR( r R + 1)}{\sqrt{n}}
    \end{align*}
    Combining with the function value bound, we get the desired bound (\ref{eq: bound for linear regression}) by Lemma \ref{mohri}.
\end{proof}

\section{Detailed discussions of Assumption \ref{assumption: as finite}\label{section: discuss assumption 1}}
In this section, we add more discussions on the Assumption \ref{assumption: as finite} and provide some concrete examples for the stochastic processes that satisfy the assumption.
We first write down the complete description for the Kolmogorov continuity theorem.
\begin{lemma}[Kolmogorov]\label{lemma: Kolmogorov}
    Let $\{X_t\}_{t \geq 0}$ be a real-valued stochastic process such that there exists positive constants $\alpha, \beta, C$ satisfying 
    \begin{equation*}
        \mathbb{E} \left[\left|X_t - X_s\right|^\alpha\right] \leq C |t-s|^{1+\beta}
    \end{equation*}
    for all $s, t \geq 0$.
    Then $X$ has a continuous modification which, with probability one, is locally {$\gamma$}-Hölder continuous for every $0 < \gamma < \beta / \alpha$. 
\end{lemma}
In the case of Brownian motion on $\mathbb{R}$, the choice of constants $\alpha=4, \beta=1, C=2$ will work in the Kolmogorov continuity theorem.
When it comes to the Gaussian process, we have the following theorem \citep[Theorem 1.]{azmoodeh2014necessary} that gives a necessary and sufficient condition for Hölder continuity.
\begin{lemma}\label{lemma: gp lipschitz}
    A centered (mean zero) Gaussian process $X$ is Hölder continuous of any order $a<H$, i.e.,
    \begin{equation*}
        |X_t - X_s| \leq C_\varepsilon |t-s|^{H-\varepsilon}, \quad \forall \varepsilon \in (0, H)
    \end{equation*}
    if and only if 
    there exists constants $c_\varepsilon$ such that
    \begin{equation*}
        \mathbb{E} \left[(X_t-X_s)^2\right] \leq c_\varepsilon (t-s)^{2H-2\varepsilon}, \quad
        \forall \varepsilon \in (0, H)
    \end{equation*}
\end{lemma}
For a stationary Gaussian process with covariance $K(s-t)$, the Hölder continuity (in expectation) assumption is equivalent to 
$1 - K(s-t)/K(0) \leq {c}_\alpha (t-s)^{2\alpha}/2$ for any $\alpha \in (0, H)$.
Now combining these results, we see that for any stationary Gaussian process with continuous mean $\mu(t)$, covariance $K(s-t)$, and  $1 - K(s-t)/K(0) \leq {c}_\alpha (t-s)^{2\alpha}/2, \forall \alpha \in (0, H)$, it satisfies Hölder continuity in Assumption \ref{assumption: as finite}.
As for the sub-Gaussian property, since the normalized Gaussian process $\Tilde{X}_t$ is standard normal at each time $t$, then any Gaussian process that satisfies Hölder continuity automatically satisfies the sub-Gaussian property in Assumption \ref{assumption: as finite}.
Concrete examples include: 
\begin{itemize}
    \item 
    identical sequences: 
    $x(t) = x$ for all $t \in [0, T]$, where $x \sim \mathcal{N}(0, 1)$
    \item 
    Gaussian white noise: 
    $\mu(t) = 0$, 
    $K(s, t) = \frac{1}{|b| \sqrt{\pi}} e^{-((s-t)/b)^2}$
    for some $b \neq 0$
    \item 
    Ornstein-Uhlenbeck process: 
    $\mu(t) = 0$, 
    $K(s, t) = e^{-|s-t|}$
\end{itemize}
\textbf{Relaxations of Assumption \ref{assumption: as finite}.}\
In fact, Assumption \ref{assumption: as finite} is used to show upper bounds for two key terms (\ref{eq: two terms that need to bound}) in the proof of Theorem \ref{thm: p=1, q=inf}.
In particular, the sub-Gaussian property in Assumption \ref{assumption: as finite} guarantees that the input random process is bounded in a finite time set with high probability.
The Hölder condition then ensures the boundedness in a infinite time set $t \in [0, T]$.
Thus, if the input random process is from a finite subset of $\mathbb{R}$, then the Hölder condition can be removed.
For example, in computer vision tasks when the input image is flattened as a sequence, the range for each pixel value is a finite set (for a MNIST image, each pixel value is a positive integer between $0$ to $255$).
In that case, the Holder continuity condition in Assumption \ref{assumption: as finite} can be dropped.

\section{Derivations for (\ref{measure for left zero padding}) and (\ref{measure for right zero padding}) in Section \ref{section: generalization bound}}
\label{appendix, derivation for padding}

For the left zero padding transformation, the key term in (\ref{thm: bound}) becomes 
\begin{align*}
    & \int_0^{2T} \left| {\rho}_\theta(2T-t)\right| \sqrt{K_1(t, t)} d t +
    \left|\int_0^{2T}  {\rho}_\theta(2T-t) \mu_1(t)
    d t\right|+1 \\
    = & 
    \int_0^{T} \left| {\rho}_\theta(T-t)\right| \sqrt{K(t, t)} d t +
    \left|\int_0^{T}  {\rho}_\theta(T-t) \mu(t)
    d t\right|+1
\end{align*}

For the right zero padding transformation, the key term in (\ref{thm: bound}) becomes 
\begin{align*}
    & \int_0^{2T} \left| {\rho}_\theta(2T-t)\right| \sqrt{K_2(t, t)} d t +
    \left|\int_0^{2T}  {\rho}_\theta(2T-t) \mu_2(t)
    d t\right|+1 \\
    = & 
    \int_0^{T} \left| {\rho}_\theta(2T-t)\right| \sqrt{K(t, t)} d t +
    \left|\int_0^{T}  {\rho}_\theta(2T-t) \mu(t)
    d t\right|+1 \\
    = & 
    \int_0^{T} \left|C e^{AT} e^{A(T-t)}B\right| \sqrt{K(t, t)} d t +
    \left|\int_0^{T} C e^{AT} e^{A(T-t)}B \mu(t)
    d t\right|+1
\end{align*}

Then we get (\ref{measure for left zero padding}) and (\ref{measure for right zero padding}).

\section{Proof for Theorem \ref{thm: p=1, q=inf}}\label{section: proof for thm1}

In this section, we will prove Theorem \ref{thm: p=1, q=inf}.
Before moving into the formal proof, we first introduce some useful lemmas that help to build the proof.

The first lemma is the Massart Lemma for the Rademacher complexity with finite class.
\begin{lemma}[Massart]\label{lemma: massart}
    Let $\mathcal{A}$ be some finite subset of $R^{m}$ and $\sigma_1,\ldots,\sigma_m$ be independent Rademacher random variables. Let $r = \sup_{a \in \mathcal{A}} \|a\|$. Then, we have,
    \begin{equation*}
        \mathbb{E}_\sigma \left[\sup_{a\in \mathcal{A}} \sum_{i=1}^m \sigma_i a_i\right] \leq r \sqrt{2 \log|\mathcal{A}|}
    \end{equation*}
\end{lemma}

The second lemma is to bound the supremum of a stochastic process that is Hölder continuous and sub-Gaussian.
\begin{lemma}[Hölder maximal inequality]\label{lemma: Hölder maximal}
    Suppose $\{X_t\}_{t \in [0, T]}$ is a centered Hölder process, i.e., 
    $\exists L, H > 0, s.t. |X_s - X_t| \leq L |s-t|^{H}, \forall s, t \in [0, T]$. 
    If further $X_t$ is 
    $\sigma^2$-sub-Gaussian for every $t \in [0, T]$, i.e., 
    $\forall u > 0, P \left(|X_t| \geq  u\right) \leq 2\exp (-u^2/2\sigma^2)$ for some $\sigma > 0$.
    Then with probability at least $1-\delta$,
    \begin{equation*}
        \sup_{t \in [0, T]} |X_t| \leq
        L + \sigma \sqrt{2 \log (2T/\delta)}.
    \end{equation*}
\end{lemma}
\begin{proof}
    The proof is based on the $\varepsilon$-net and covering number argument.
    We first discretize the time interval $[0, T]$ into $N$ parts $[0, T/N] \cup [T/N, 2T/N] \cdots \cup [(N-1)T/N, T]$.
    Then for any time $t \in [0, T]$, there exists a $\pi(t) \in \{0, T/N, \ldots, (N-1)T/N\}$ such that $|t - \pi(t)| \leq T/N$.
    Therefore, by Hölder continuity, we have 
    \begin{equation*}
        \sup_{t \in [0, T]} |X_t| 
        \leq 
        \sup_{t \in [0, T]} |X_t - X_{\pi(t)}| + \sup_{t \in [0, T]} |X_{\pi(t)}|
        \leq 
        L \left(\frac{T}{N}\right)^H +
        \max_{i \in [0:N-1]} 
        |X_{iT/N}|.
    \end{equation*}
    Since $X_t$ is sub-Guassian for every time $t \in [0, T]$, then for each $i \in [0:N-1]$, by letting $u = \sigma \sqrt{2 \log (2N/\delta)}$, we have 
    with probability at least $1 - \delta/N$,
    \begin{equation*}
        X_{iT/N} \leq \sigma \sqrt{2 \log (2N/\delta)}.
    \end{equation*}
    Taking the union bound over all $i \in [0:N-1]$, we have 
    with probability at least $1 - \delta$,
     \begin{equation*}
         \max_{i \in [0:N-1]} 
        X_{iT/N} \leq \sigma \sqrt{2 \log (N/\delta)}.
     \end{equation*}
     Hence,
     \begin{equation*}
         \sup_{t \in [0, T]} X_t 
        \leq  L \left(\frac{T}{N}\right)^H + \sigma \sqrt{2 \log (2N/\delta)}
     \end{equation*}
     holds for all $N$.
     Here we simply take $N = [T] + 1$, then we get 
     \begin{equation*}
         \sup_{t \in [0, T]} X_t 
        \leq  L + \sigma \sqrt{2 \log (2T/\delta)}.
     \end{equation*}
\end{proof}

Now we are ready to prove the main result Theorem \ref{thm: p=1, q=inf}.

\begin{proof}
We let {$g_{\theta}(x) :=  \int_0^T  {\rho}_\theta(T-t) x(t) d t - y$}, then
the generalization gap is given by
\begin{equation*}
    R_{x}(\theta) -  {R}_n(\theta) =
    \mathbb{E}_{x} [g_{\theta}^2(x)]
    - \frac{g_{\theta}^2(x_1) + \ldots + g_{\theta}^2(x_n)}{n}.
\end{equation*}

Now let hypothesis space $\mathcal{F} = \{x \mapsto g_{\theta}^2(x) : \theta \in \Theta\}$, then its empirical Rademacher complexity is given by
\begin{align*}
    \mathcal{R}_S (\mathcal{F}) & = \mathbb{E}_{\sigma} \left[\sup_{\theta \in \Theta} \frac{1}{n} \sum_{i=1}^n \sigma_i g_{\theta}^2(x_i)\right] \\
    & = 
    \frac{1}{n}
    \mathbb{E}_{\sigma} \left[\sup_{\theta \in \Theta}
    \sum_{i=1}^n \sigma_i
    \left|\int_0^T  {\rho}_\theta(T-t) x_i(t) d t - y_i\right|^2\right] 
\end{align*}

By the Talagrand’s contraction Lemma \ref{lemma: contraction}, since $g_{\theta}^2(x_i)$ is $2 \sup_{\theta \in \Theta, i \in [1:n]} |g_{\theta}(x_i)|$ Lipschitz, we have 
\begin{align*}
    \mathcal{R}_S (\mathcal{F})
    & \leq  
    2 \sup_{\theta \in \Theta, i \in [1:n]} |g_\theta(x_i)| \cdot  
    \frac{1}{n}
    \mathbb{E}_{\sigma} \left[\sup_{\theta \in \Theta}
    \sum_{i=1}^n \sigma_i
    \left(\int_0^T  {\rho}_\theta(T-t) x_i(t) d t - y_i\right)\right]
    \\
    & = 
    \frac{2 \sup_{\theta \in \Theta, i \in [1:n]} |g_\theta(x_i)|}{n}
    \mathbb{E}_{\sigma} \left[\sup_{\theta \in \Theta}
    \int_0^T  {\rho}_\theta(T-t)  \sum_{i=1}^n \sigma_i x_i(t) d t\right]
\end{align*}

Now we separate the expectation into two parts: the unbiased part invovled with $x_i(t) - \mu(t)$ and the biased part $\mu(t)$, by noticing that 
\begin{align*}
    & \mathbb{E}_{\sigma} \left[\sup_{\theta \in \Theta}
    \int_0^T  {\rho}_\theta(T-t)  \sum_{i=1}^n \sigma_i x_i(t) d t\right] \\
    = &
    \mathbb{E}_{\sigma} \left[\sup_{\theta \in \Theta}
    \int_0^T  {\rho}_\theta(T-t)  \sum_{i=1}^n \sigma_i (x_i(t) - \mu(t)) d t + 
    \int_0^T  {\rho}_\theta(T-t)  \sum_{i=1}^n \sigma_i \mu(t)
    d t
    \right] \\
    \leq & 
    \mathbb{E}_{\sigma} \left[\sup_{\theta \in \Theta}
    \int_0^T  {\rho}_\theta(T-t)  \sum_{i=1}^n \sigma_i (x_i(t) - \mu(t)) d t\right] + 
    \mathbb{E}_{\sigma}
    \left[\sup_{\theta \in \Theta}
    \int_0^T  {\rho}_\theta(T-t)  \sum_{i=1}^n \sigma_i \mu(t)
    d t
    \right] 
\end{align*}

For the unbiased part, 
by the Hölder's inequality, for any $p, q \in [1, \infty]$ such that $\frac{1}{p} + \frac{1}{q} = 1$, 
\begin{equation}\label{eq: unbiased}
    \begin{aligned}
    & \mathbb{E}_\sigma \left[\sup_{\theta \in \Theta}
    \int_0^T  {\rho}_\theta(T-t)  \sum_{i=1}^n \sigma_i (x_i(t) - \mu(t)) d t\right] \\
    \leq &
    \sup_{\theta \in \Theta} \left(\int_0^T \left| {\rho}^p_\theta(T-t)\right|
    K^{p/2}(t,t)
    d t\right)^{1/p} \mathbb{E}_{\sigma} \left[\left(\int_0^T \left|\sum_{i=1}^n \sigma_i \frac{x_i(t) - \mu(t)}{\sqrt{K(t,t)}}\right|^q d t\right)^{1/q}\right]
    \end{aligned}
\end{equation}

For the biased part, 
\begin{equation}\label{eq: biased}
    \begin{aligned}
    \mathbb{E}_{\sigma}
    \left[\sup_{\theta \in \Theta}
    \int_0^T  {\rho}_\theta(T-t)  \sum_{i=1}^n \sigma_i \mu(t)
    d t
    \right] & \leq 
    \sup_{\theta \in \Theta} \left|\int_0^T  {\rho}_\theta(T-t) \mu(t)
    d t\right|
    \mathbb{E}_{\sigma}
    \left[\left|\sum_{i=1}^n \sigma_i\right|\right] \\
    & \leq 
    \sup_{\theta \in \Theta} \left|\int_0^T  {\rho}_\theta(T-t) \mu(t)
    d t\right|
    \sqrt{\mathbb{E}_{\sigma}
    \left[\left|\sum_{i=1}^n \sigma_i\right|^2\right]} \\
    & = 
    \sqrt{n}
    \sup_{\theta \in \Theta} \left|\int_0^T  {\rho}_\theta(T-t) \mu(t)
    d t\right|
    \end{aligned}
\end{equation}

Now for the unbiased part (\ref{eq: unbiased}), we take $p=1, q=\infty$.
Then we have 
\begin{equation}\label{eq: further unbiased}
    \begin{aligned}
        & \mathbb{E}_\sigma \left[\sup_{\theta \in \Theta}
    \int_0^T  {\rho}_\theta(T-t)  \sum_{i=1}^n \sigma_i (x_i(t) - \mu(t)) d t\right] \\
    \leq &
    \sup_{\theta \in \Theta} \left(\int_0^T \left| {\rho}_\theta(T-t)\right|
    \sqrt{K(t, t)}
    d t\right) 
    \mathbb{E}_{\sigma} \left[\sup_{t \in [0, T]} \left|\sum_{i=1}^n \sigma_i \frac{x_i(t) - \mu(t)}{\sqrt{K(t,t)}}\right|\right]
    \end{aligned}
\end{equation}

Also by the same argument, note that 
\begin{equation}\label{eq: sup g_theta}
    \begin{aligned}
        & \sup_{\theta \in \Theta, i \in [1:n]} |g_\theta(x_i)| \\
    = & \sup_{\theta \in \Theta, i \in [1:n]}  
    \left|\int_0^T  {\rho}_\theta(T-t) x_i(t) d t - y_i\right| \\
    \leq &
    \sup_{\theta \in \Theta, i \in [1:n]} 
    \left|\int_0^T  {\rho}_\theta(T-t) (x_i(t)-\mu(t)) d t\right| + 
    \sup_{\theta \in \Theta}
    \left|\int_0^T  {\rho}_\theta(T-t) \mu(t) d t\right| 
    +
    1 \\
    \leq &
    \sup_{\theta \in \Theta} 
    \left(\int_0^T \left| {\rho}_\theta(T-t)\right| \sqrt{K(t,t)} d t\right) 
    \sup_{i \in [1:n], t \in [0, T]} 
    \left|\frac{x_i(t)-\mu(t)}{\sqrt{K(t,t)}}\right|
    + 
    \sup_{\theta \in \Theta}
    \left|\int_0^T \Re( {\rho}_\theta(T-t)) \mu(t) d t\right| 
    +
    1
    \end{aligned}
\end{equation}

Thus, there are two terms that we need to bound:
\begin{equation}\label{eq: two terms that need to bound}
    \sup_{i \in [1:n], t \in [0, T]} 
    \left|\frac{x_i(t)-\mu(t)}{\sqrt{K(t,t)}}\right|, 
    \quad
    \mathbb{E}_{\sigma} \left[\sup_{t \in [0, T]} \left|\sum_{i=1}^n \sigma_i \frac{x_i(t)-\mu(t)}{\sqrt{K(t,t)}}\right|\right]
\end{equation}

For the first term, notice that the normalized Gaussian process $\frac{x_i(t)-\mu(t)}{\sqrt{K(t,t)}}$ is centered. 
By Assumption \ref{assumption: as finite}, it is Hölder continuous and $\sigma^2$-sub-Gaussian on $t \in [0, T]$.
Therefore, we can directly apply Lemma \ref{lemma: Hölder maximal} and get with probability at least $1-\delta/3n$,
\begin{equation*}
    \sup_{t\in[0,T]} \left|\frac{x_i(t)-\mu(t)}{\sqrt{K(t,t)}}\right| 
    \leq
    L + \sigma \sqrt{2 \log (6Tn/\delta)}, 
    \quad
    \forall i =1,\ldots,n
\end{equation*}
Now by taking a union bound over $i = 1,\ldots,n$, we get with probability at least $1-\delta/3$,
\begin{equation}\label{eq: sup of gp}
    \sup_{i \in [1:n], t \in [0, T]} 
    \left|\frac{x_i(t)-\mu(t)}{\sqrt{K(t,t)}}\right| 
    \leq 
    L + \sigma \sqrt{2 \log (6Tn/\delta)}.
\end{equation}
For the second term, we apply the $\varepsilon$-net and covering number argument as in Lemma \ref{lemma: Hölder maximal}.
We discretize the time interval $[0, T]$ into $N$ parts $[0, T/N] \cup [T/N, 2T/N] \cdots \cup [(N-1)T/N, T]$, then for any $t \in [0, T]$, there exists a sub-interval such that $t \in [(k-1)T/N, kT/N]$ for some $k \in [1:N]$.
Therefore, $\forall t \in [0, T]$ such that $t \in [(k-1)T/N, kT/N]$ for some $k \in [1:N]$, by Hölder continuity in Assumption \ref{assumption: as finite} for the normalized process, we have 
\begin{align*}
    \left|\sum_{i=1}^n \sigma_i \frac{x_i(t)-\mu(t)}{\sqrt{K(t,t)}}\right| & \leq 
    \left|\sum_{i=1}^n \sigma_i \frac{x_i\left(\frac{(k-1)T}{N}\right)-\mu\left(\frac{(k-1)T}{N}\right)}{\sqrt{K\left(\frac{(k-1)T}{N}, \frac{(k-1)T}{N}\right)}}\right| + 
    \left|\sum_{i=1}^n \sigma_i \left(\frac{x_i\left(\frac{(k-1)T}{N}\right)-\mu\left(\frac{(k-1)T}{N}\right)}{\sqrt{K\left(\frac{(k-1)T}{N}, \frac{(k-1)T}{N}\right)}} - 
    \frac{x_i(t)-\mu(t)}{\sqrt{K(t,t)}}\right)\right| \\
    & \leq 
    \max_{k=1,\ldots,N}
    \left|\sum_{i=1}^n \sigma_i \frac{x_i\left(\frac{(k-1)T}{N}\right)-\mu\left(\frac{(k-1)T}{N}\right)}{\sqrt{K\left(\frac{(k-1)T}{N}, \frac{(k-1)T}{N}\right)}}\right| +
    \|\sigma\| \sqrt{n} L \left(\frac{T}{N}\right)^{H} \\
    & = 
    \max_{k=1,\ldots,N}
    \left|\sum_{i=1}^n \sigma_i \frac{x_i\left(\frac{(k-1)T}{N}\right)-\mu\left(\frac{(k-1)T}{N}\right)}{\sqrt{K\left(\frac{(k-1)T}{N}, \frac{(k-1)T}{N}\right)}}\right| +
    nL \left(\frac{T}{N}\right)^{H}
\end{align*}
Then by the Massart Lemma \ref{lemma: massart} and the sup norm bound (\ref{eq: sup of gp}), with probability at least $1-\delta/3$,
\begin{align*}
    \mathbb{E}_{\sigma} \left[\sup_{t \in [0, T]} \left|\sum_{i=1}^n \sigma_i \frac{x_i(t)-\mu(t)}{\sqrt{K(t,t)}}\right|\right] & \leq 
    \mathbb{E}_{\sigma}
    \left[\max_{k=1,\ldots,N}
    \left|\sum_{i=1}^n \sigma_i \frac{x_i\left(\frac{(k-1)T}{N}\right)-\mu\left(\frac{(k-1)T}{N}\right)}{\sqrt{K\left(\frac{(k-1)T}{N}, \frac{(k-1)T}{N}\right)}}\right|
    \right] +
    nL \left(\frac{T}{N}\right)^{H} \\
    & \leq 
    \sqrt{2 n \log N} \cdot 
    \sup_{i \in [1:n], t \in [0, T]}
    \left|\frac{x_i(t)-\mu(t)}{\sqrt{K(t,t)}}\right|+
    nL \left(\frac{T}{N}\right)^{H} \\
    & \leq 
    \sqrt{2 n \log N}
    \left(L + \sigma \sqrt{2 \log (6Tn/\delta)}\right)
    +
    nL \left(\frac{T}{N}\right)^{H}
\end{align*}
Since $N$ is an arbitrary integer number, we let $N = \left[T n^{1/H}\right] + 1$, then we get 
\begin{equation}\label{eq: rademacher process}
    \begin{aligned}
        \mathbb{E}_{\sigma} \left[\sup_{t \in [0, T]} \left|\sum_{i=1}^n \sigma_i \frac{x_i(t)-\mu(t)}{\sqrt{K(t,t)}}\right|\right] & \leq 
        {\mathcal{O}} 
        \left(
        \sqrt{n \cdot \log N \cdot \log(Tn/\delta)}
        \right) \\
        & \leq 
        \mathcal{O} 
        \left(
        \sqrt{n} \log (N T n/\delta)
        \right) \\
        & = \mathcal{O} 
        \left(\sqrt{n} \log (Tn /  \delta)\right).
    \end{aligned}
\end{equation}
Combining (\ref{eq: sup of gp}), (\ref{eq: rademacher process}), (\ref{eq: biased}) and (\ref{eq: further unbiased}), we can further bound (\ref{eq: sup g_theta}) as
\begin{equation}\label{eq: further g_theta sup}
    \sup_{\theta \in \Theta, i \in [1:n]} |g_\theta(x_i)| \leq 
    \sup_{\theta \in \Theta} 
    \left(\int_0^T \left| {\rho}_\theta(T-t)\right| \sqrt{K(t,t)} d t\right) 
    \mathcal{O} \left(\sqrt{\log(Tn/\delta)}\right)
    + 
    \sup_{\theta \in \Theta}
    \left|\int_0^T \Re( {\rho}_\theta(T-t)) \mu(t) d t\right| 
    +
    1
\end{equation}
And the Rademacher complexity is further bounded as 
\begin{align*}
    & \mathcal{R}_S (\mathcal{F}) \\
    \leq &  
    \frac{2 \sup_{\theta \in \Theta, i \in [1:n]} |g_\theta(x_i)|}{n}
    \mathbb{E}_{\sigma} \left[\sup_{\theta \in \Theta}
    \int_0^T  {\rho}_\theta(T-t)  \sum_{i=1}^n \sigma_i x_i(t) d t\right] \\
    \leq & 
    \frac{2 \sup_{\theta \in \Theta, i \in [1:n]} |g_\theta(x_i)|}{n}
    \left(\sup_{\theta \in \Theta} \int_0^T \left| {\rho}_\theta(T-t)\right|
    \sqrt{K(t, t)}
    d t 
    + 
    \sup_{\theta \in \Theta} \left|\int_0^T  {\rho}_\theta(T-t) \mu(t)
    d t\right|\right)
    \cdot
    \mathcal{O} 
    \left(\sqrt{n} \log (Tn /  \delta)\right) \\
    \leq & 
    \left(\sup_{\theta \in \Theta} \int_0^T \left| {\rho}_\theta(T-t)\right|
    \sqrt{K(t, t)}
    d t 
    + 
    \sup_{\theta \in \Theta} \left|\int_0^T  {\rho}_\theta(T-t) \mu(t)
    d t\right|+1\right)^2
    \cdot 
    \mathcal{O} \left(\frac{\log^{3/2}(Tn/\delta)}{\sqrt{n}}\right).
\end{align*}
Finally, by the symmetrization of $R_x(\theta) -  {R}_n(\theta)$,  combining it with (\ref{eq: further g_theta sup}) and (\ref{mohri}), we have with probability at least $1-\delta$, 
\begin{equation*}
    \sup_{\theta \in \Theta} \left|R_x(\theta) -  {R}_n(\theta)\right| \leq 
    \left(\sup_{\theta \in \Theta} \int_0^T \left| {\rho}_\theta(T-t)\right|
    \sqrt{K(t, t)}
    d t 
    + 
    \sup_{\theta \in \Theta} \left|\int_0^T  {\rho}_\theta(T-t) \mu(t)
    d t\right|+1\right)^2
    \cdot 
    \mathcal{O} \left(\frac{\log^{3/2}(Tn/\delta)}{\sqrt{n}}\right).
\end{equation*}
\end{proof}

\section{Proof for Proposition \ref{prop: initialization}}
\label{appendix: proof of proposition 1}

\begin{proof}
First, notice that by the Hölder's inequality with $p=1, q=\infty$, we have 
\begin{align*}
    & \mathbb{E}_x \left[\left|\int_0^T  {\rho}_{\Tilde{\theta}}(T-t) x(t) d t\right|\right] \\
    = & \frac{\mathbb{E}_x
    \left[\left|\int_0^T  {\rho}_\theta(T-t) x(t) d t\right|\right]}{\int_0^T \left| {\rho}_\theta(T-t)\right| \sqrt{K(t, t)}d t + \left|\int_0^T  {\rho}_\theta(T-t) \mu(t)d t\right|} \\
    \leq & 
    \frac{\mathbb{E}_x
    \left[\left|\int_0^T  {\rho}_\theta(T-t) (x(t) - \mu(t)) d t\right|\right] + \left|\int_0^T  {\rho}_\theta(T-t)\mu(t)dt\right|}{\int_0^T \left| {\rho}_\theta(T-t)\right| \sqrt{K(t, t)}d t + \left|\int_0^T  {\rho}_\theta(T-t) \mu(t)d t\right|}
    \\
    \leq & 
    \frac{\int_0^T | {\rho}_\theta(T-t)| \sqrt{K(t,t)}dt
    \cdot 
    \mathbb{E}_x \left[\sup_{t \in [0, T]} \left|\frac{x(t)-\mu(t)}{\sqrt{K(t,t)}}\right|\right] + 
    \left|\int_0^T  {\rho}_\theta(T-t)\mu(t)dt\right|}{\int_0^T \left| {\rho}_\theta(T-t)\right| \sqrt{K(t, t)}d t + \left|\int_0^T  {\rho}_\theta(T-t) \mu(t)d t\right|}
    \\
    \leq &  
    \mathbb{E}_x \left[\sup_{t \in [0, T]} \left|\frac{x(t)-\mu(t)}{\sqrt{K(t,t)}}\right|\right] + 1
\end{align*}

We let $X_t := \frac{x(t)-\mu(t)}{\sqrt{K(t,t)}}$, then by Assumption \ref{assumption: as finite}, $X_t$ is
Hölder continuous and $\sigma^2$ sub-Gaussian for any $t \in [0, T]$.
Again, we use an $\varepsilon$-net argument to bound $\mathbb{E} \left[\sup_{t \in [0, T]} |X_t|\right]$.
By separating the time interval $[0, T]$ into $N$ parts $[0, T/N] \cup [T/N, 2T/N] \cdots \cup [(N-1)T/N, T]$.
Then for any time $t \in [0, T]$, there exists a $\pi(t) \in \{0, T/N, \ldots, (N-1)T/N\}$ such that $|t - \pi(t)| \leq T/N$.
Therefore, by Hölder continuity,
\begin{align*}
    \mathbb{E} \left[\sup_{t \in [0, T]} |X_t|\right] 
    & \leq 
    \mathbb{E} \left[\sup_{t \in [0, T]} |X_t - X_{\pi(t)}|\right] + \mathbb{E}\left[\sup_{t \in [0, T]} |X_{\pi(t)}|\right] \\
    & \leq 
    L \left(\frac{T}{N}\right)^H +
    \mathbb{E} \left[\max_{i \in [0:N-1]} 
    |X_{iT/N}|\right].
\end{align*}
For the maximum over a finite class,
notice that for any $u_0 > 0$,
\begin{align*}
    \mathbb{E} \left[\max_{i \in [0:N-1]} 
    |X_{iT/N}|\right] & = 
    \int_0^\infty
    P \left(\max_{i \in [0:N-1]} 
    |X_{iT/N}| \geq u\right)du\\
    & =
    \int_0^{u_0}
    P \left(\max_{i \in [0:N-1]} 
    |X_{iT/N}| \geq u\right)du + 
    \int_{u_0}^\infty
    P \left(\max_{i \in [0:N-1]} 
    |X_{iT/N}| \geq u\right)du \\
    & \leq u_0 + \int_{u_0}^\infty
    \sum_{i=0}^{N-1} 
    P \left(|X_{iT/N}| \geq u\right)du.
\end{align*}
Since $X_{iT/N}$ is $\sigma^2$ sub-Gaussian for every $i \in [0:N-1]$, then $\forall u_0 >0$, 
\begin{align*}
    \mathbb{E} \left[\max_{i \in [0:N-1]} 
    |X_{iT/N}|\right] & \leq u_0 + 2N \int_{u_0}^\infty
    \exp \left(-\frac{u^2}{2\sigma^2}\right)du \\
    & \leq  u_0 + 2N
    \int_{u_0}^\infty \frac{u}{u_0}
    \exp \left(-\frac{u^2}{2\sigma^2}\right)du \\
    & = u_0 + \frac{2N \sigma^2}{u_0} \exp \left(-\frac{u_0^2}{2\sigma^2}\right).
\end{align*}
Minimizing the above term over $u_0>0$, we can simply let $u_0 = \sigma \sqrt{2\log2N}$, then 
\begin{equation*}
    \mathbb{E} \left[\max_{i \in [0:N-1]} 
    |X_{iT/N}|\right]  \leq 
    \sigma \sqrt{2\log2N}
    + \frac{\sigma}{\sqrt{2\log2N}} \leq 2\sigma \sqrt{2\log2N}.
\end{equation*}
Now back to the original upper bound, we get 
\begin{equation*}
    \mathbb{E} \left[\sup_{t \in [0, T]} |X_t|\right] 
    \leq 
    L \left(\frac{T}{N}\right)^H +
    2\sigma \sqrt{2\log2N}.
\end{equation*}
Since $N$ is an arbitrary positive integer, we simply take $N = [T]+1$, finally we get 
\begin{equation*}
    \mathbb{E} \left[\sup_{t \in [0, T]} |X_t|\right] 
    \leq 
    L + 2 \sigma \sqrt{2 \log (2T+2)} = \mathcal{O} (\sqrt{\log T}).
\end{equation*}
\end{proof}

\section{Lipschitz function of sub-Gaussian random variables}
\label{appendix: lip for subgaussian}
In this section, we provide some known examples for the sub-Gaussian random variables that remain the sub-Gaussian property under a Lipschitz function.
\begin{enumerate}
    \item For a bounded random variable $X \in [0, 1]$, if $f : [0, 1] \xrightarrow{} \mathbb{R}$ is a quasi-convex function, i.e., $\{x : f(x) \leq s\}$ is a convex set for all $s \in \mathbb{R}$.
    If $f$ is further Lipschitz in $[0, 1]$, then $f(X)$ is sub-Gaussian.
    See Theorem 7.12 in \citet{boucheron2013concentration}.
    \item 
    For a sub-Gaussian random variable $X$ that has density of the form $e^{-U(x)}$ with $U$ being twice continuously differentiable and $U''(x) > 0, \forall x \in \mathbb{R}$, then if $f$ is a Lipschitz function, $f(X)$ is also sub-Gaussian. 
    See Theorem 5.2.15 in \citet{vershynin2020high}.
\end{enumerate}

\end{document}